\documentclass{article}

\usepackage{microtype}
\usepackage{graphicx}
\usepackage{subcaption}
\usepackage{booktabs} 
\usepackage{hyperref}

\usepackage[accepted]{icml2026}

\usepackage{amsmath}
\usepackage{amssymb}
\usepackage{mathtools}
\usepackage{amsthm}
\usepackage{comment}

\usepackage[capitalize,noabbrev]{cleveref}

\theoremstyle{plain}
\newtheorem{theorem}{Theorem}[section]

\newtheorem{lemma}[theorem]{Lemma}

\theoremstyle{definition}

\theoremstyle{remark}

\newtheorem*{remark_non}{Remark}

\usepackage[textsize=tiny]{todonotes}

\usepackage{booktabs}

\newcommand{\sh}{\textcolor{black}}

\newcommand{\E}[1]{\mathbb{E}{\left[#1\right]}}

\newcommand{\C}{\ensuremath{\tfrac{1}{0.5+\sqrt{1+k^2} \cdot \max_{\hat{p} \in [0,1]}\epsilon(\hat{p})}}}
\usepackage{bbm}
\usepackage{subcaption} 

\icmltitlerunning{Testing For Distribution Shifts with Conditional Conformal Test Martingales}

\begin{document}

\twocolumn[
  \icmltitle{Testing For Distribution Shifts with Conditional Conformal Test Martingales}



  \icmlsetsymbol{equal}{*}



  \begin{icmlauthorlist}
    \icmlauthor{Shalev Shaer}{EEIIT}
    \icmlauthor{Yarin Bar}{CSIIT}
    \icmlauthor{Drew Prinster}{CSJH}
    \icmlauthor{Yaniv Romano}{EEIIT,CSIIT}
  \end{icmlauthorlist}

  \icmlaffiliation{EEIIT}{Department of Electrical and Computer Engineering, Technion IIT, Israel}
  \icmlaffiliation{CSIIT}{Department of Computer Science, Technion IIT, Israel}
  \icmlaffiliation{CSJH}{Department of Computer Science, Johns Hopkins University, Baltimore, MD, USA}

  \icmlcorrespondingauthor{Shalev Shaer}{shalev.shaer@campus.technion.ac.il}

  \icmlkeywords{Machine Learning, ICML}

  \vskip 0.3in
]



\printAffiliationsAndNotice{}  

\begin{abstract}

We propose a sequential test for detecting arbitrary distribution shifts that allows conformal test martingales (CTMs) to work under a fixed, reference-conditional setting. Existing CTM detectors construct test martingales by continually growing a reference set with each incoming sample, using it to assess how atypical the new sample is relative to past observations. While this design yields anytime-valid type-I error control, it suffers from \emph{test-time contamination}: after a change, post-shift observations enter the reference set and dilute the evidence for distribution shift, increasing detection delay and reducing power.

In contrast, our method avoids contamination by design by comparing each new sample to a \emph{fixed} null reference dataset. Our main technical contribution is a robust martingale construction that remains valid conditional on the null reference data, achieved by explicitly accounting for the estimation error in the reference distribution induced by the finite reference set. This yields anytime-valid type-I error control together with guarantees of asymptotic power one and bounded expected detection delay. Empirically, our method detects shifts faster than standard CTMs, providing a powerful and reliable distribution-shift detector.

\end{abstract}
\section{Introduction}
\label{sec:intro}

Detecting when a data stream deviates from a reference distribution is critical for maintaining the reliability of deployed systems, from medical monitoring to computer vision and systems powered by large language models. For example, the performance of machine learning systems can degrade when a pre-trained model encounters new environments at test time. In such cases, monitoring tools can be used to detect when the test-time inputs have shifted relative to the training distribution, triggering actions such as model retraining, additional data collection, or human intervention.

To formalize the problem, let $\{Z_1^{\text{in}}, ..., Z_n^{\text{in}}\}$ denote a set of $n$ i.i.d., possibly high-dimensional reference samples from some unknown distribution $P_Z$, and let $\hat{s}$ denote some function used for scoring or evaluating the data, such as the confidence of a deployed classifier on a test image \citep{hendrycks2017baseline,barprotected}, the likelihood of observed covariates under a pre-trained density estimator \citep{ren2019likelihood}, the verbalized confidence of a large language model \citep{linteaching}, etc. As is common in the literature, we assume the output of $\hat{s}$ is a one-dimensional real number, and in particular denote the score for the $i$th data sample as $X_i^\text{in}:=\hat{s}(Z_i^\text{in})$, so $D_0 = \{X_1^\text{in}, \ldots, X_n^\text{in}\}$ denotes the scores for the $n$ reference samples, which constitute $n$ i.i.d. draws from an unknown distribution $P$.

Our goal is to detect whether an observed sequence of test samples $Z_t$, for $t = 1, 2, \ldots$, deviates from this reference distribution, $P_Z$.
Since the scoring function $\hat{s}$ is fixed, any distributional shift in the scores $X_1,...,X_t$ implies a shift in the raw data $Z_1,...,Z_t$. Consequently, we cast the problem as detecting deviations in the univariate sequence $X_1,...,X_t$.
Formally, given $D_0$, we aim to sequentially test the following null hypothesis:
\begin{equation}
    \label{eq:null}
    H_0: X_t \overset{\text{i.i.d.}}{\sim} P, \ \ \forall \ t\geq 1.
\end{equation}
To enable reliable monitoring in an online setting, the test must be \emph{anytime-valid}: it should control the type-I error at a user-specified level $\alpha \in (0,1)$ (e.g., $5\%$) simultaneously over all times $t$. Here, the type-I error is the probability of incorrectly rejecting $H_0$ when it is in fact true. At the same time, we aim for a \emph{powerful} test that rejects the null as quickly as possible once $H_0$ is false; for example due to a sudden shift (as in change-point detection) or a gradual drift that may begin slowly and evolve over time.

In this work, we build on \emph{testing by betting}, in which statistical evidence for distribution shift is quantified by the wealth of a gambler who repeatedly bets against the null hypothesis as data arrive. Loosely speaking, the bettor's goal is to place bets whose gains increase when new scores deviate from what would be expected under $P$. Crucially, the game is formulated such that it is unlikely that the bettor would be able to increase their wealth significantly if the null holds (i.e., there is no shift). Formally, the gambler's wealth over time defines a nonnegative martingale (or e-process), so that large values are interpreted as strong evidence against $H_0$. This idea has been instantiated in several settings, where the ones closest to our work are tests of exchangeability and change-point detection based on conformal test martingales (CTMs) 
\citep{vovk2021testing, ramdas2022testing, saha2024testing, fischer2025sequential}. 

A key feature of CTM-based methods is that they achieve anytime-valid type-I error control by continuously updating the reference set to include all observed points as they arrive, so that $D_t = D_0 \cup \{X_1, \ldots, X_t\}$. At time $t$, the bet---placed on a conformal p-value---reflects how atypical the new score $X_t$ is relative to the ones in $D_t$. We argue that computing the bets using data gathered during testing makes the CTM-based procedure inherently susceptible to ``test-time contamination.'' After a distribution shift, the reference pool $D_t$ progressively fills with post-shift observations, so later post-shift scores are evaluated relative to a mixture of null and alternative samples and can appear less atypical than the earliest post-shift scores. 

This test-time contamination phenomenon has been observed empirically \citep{podkopaev2022tracking}, and is demonstrated in Figure~\ref{fig:synth_pval_contamination}, which displays the running mean of conformal p-values (per-sample shift evidence) over time. As portrayed, the post-shift standard CTM p-values initially spike, reflecting high evidence. However, as the post-shift data contaminates the growing reference set $D_t$, the mean p-value gradually decays back towards the in-distribution level (0.5), indicating that the method progressively perceives the shifted data as normal. This contamination can dilute the evidence, reduce statistical power, and slow down detection. 
Indeed, although powerful in practice, CTM-based methods do not come with general theoretical guarantees of consistency. There is no assurance that the test will eventually reject the null in the presence of a distribution shift \citep{saha2024testing}.

\begin{figure}[t]
\centering

\includegraphics[width=0.5\textwidth]{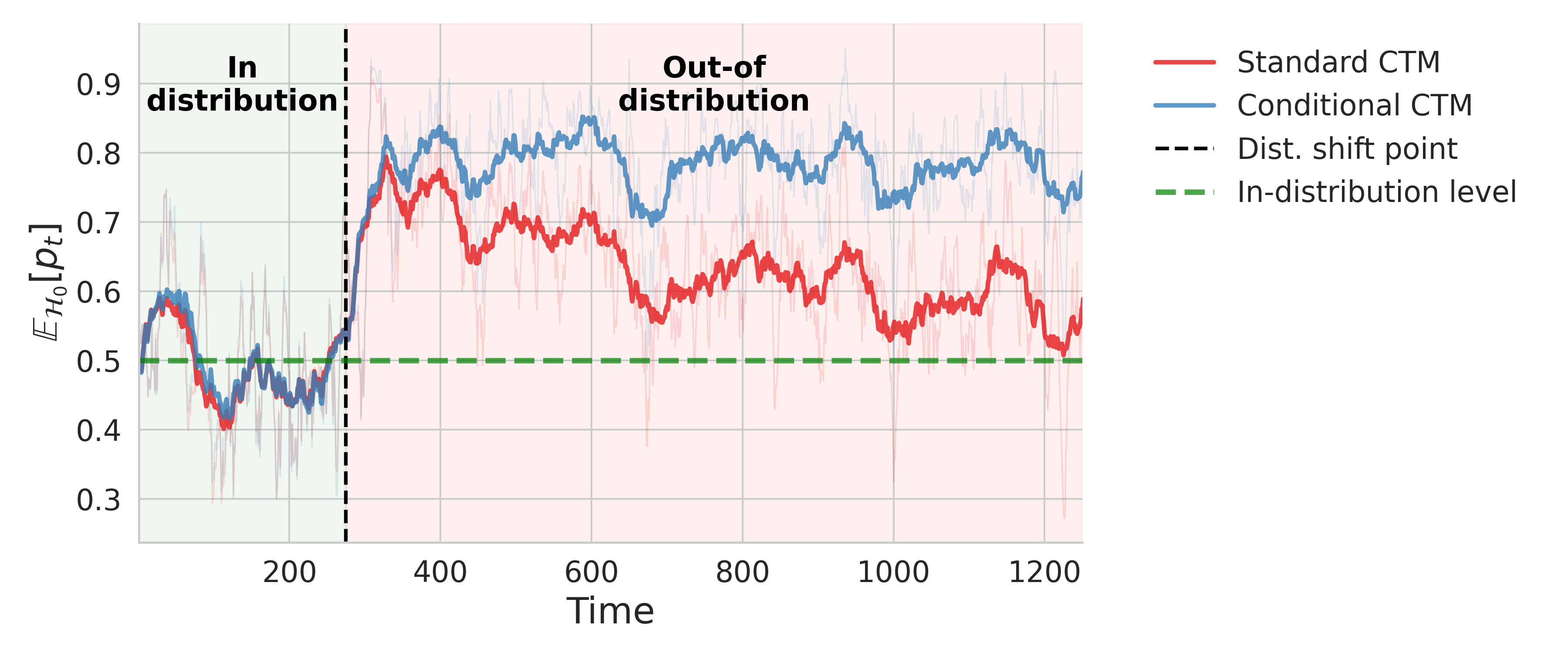}

\caption{\textbf{Running mean of conformal $p$-values over time.} 
Comparison of conformal $p$-values (mean over 10 values) under a change-point introduced at timestep $300$. 
A $p$-value near $0.5$ indicates inability to detect distributional change; 
thus, the decay of $p$-values toward $0.5$ reflects increasing contamination of $D_t$ over time. Exact details are provided in Appendix~\ref{appdx:synth_pval}. 
}
\label{fig:synth_pval_contamination} 
\end{figure}

\begin{figure}[t]
\centering
\begin{subfigure}[b]{0.245\textwidth} 
    \includegraphics[width=\textwidth]{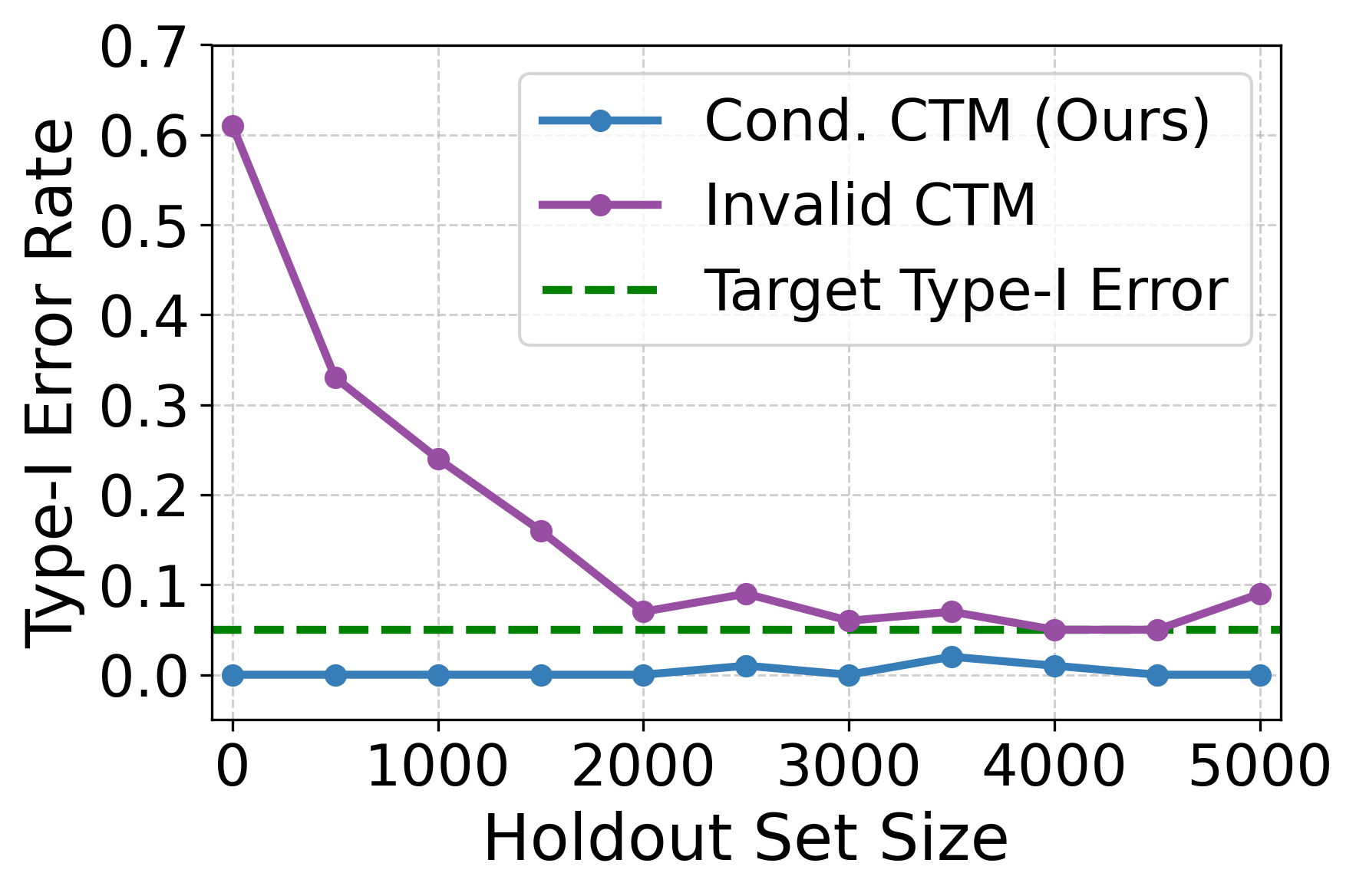}
\end{subfigure}
\begin{subfigure}[b]{0.245\textwidth} 
    \includegraphics[width=\textwidth]{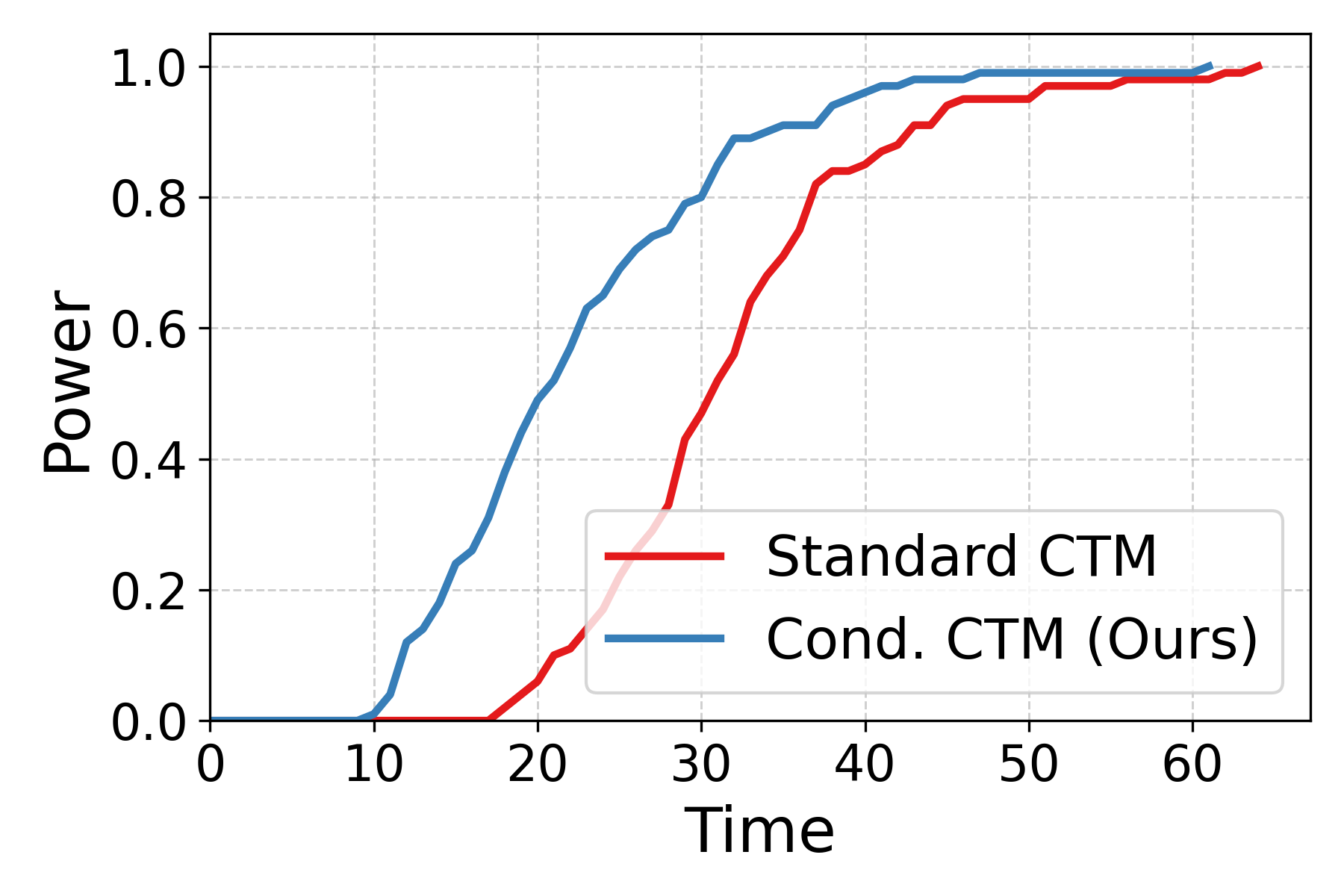}
\end{subfigure}

\caption{\textbf{Type-I error control under finite reference samples.} (left): Comparison of the Type-I error rate, evaluated over 100 repetitions, between an invalid CTM (purple) and our proposed method (blue) for varying reference set sizes $n$. The dashed line represents the test level $\alpha=0.05$. \textbf{Power under an alternative.} Cumulative power over time, evaluated over 100 repetitions, for detecting an immediate shift (at time step $0$) from $\mathcal{N}(0,1)$ to $\mathcal{N}(1,1)$, with reference size $n=2000$. Our method (blue) achieves higher power and faster detection compared to the CTM (red). Exact details are provided in Appendix~\ref{appdx:fig_1}.
}
\label{fig:motivation_combined} 
\end{figure}

\subsection*{Preview of Our Method and Key Contributions}

In this work, we also adopt the perspective of \emph{testing by betting} with CTMs. In contrast to existing CTM-based approaches that update their bets using the growing pool of data $D_t = D_0 \cup \{X_1, \ldots, X_t\}$, we consider a gambler who repeatedly places bets on the (p-values of the) incoming scores $X_1, X_2, \ldots$ using only the fixed reference data $D_0$ as a guide. As a result, our \emph{reference-conditional CTM} approach avoids test-time contamination by design: evidence for distribution shift is computed by comparing test-time scores to the fixed reference data $D_0$, and test-time observations are never incorporated into this reference set. This advantage is illustrated in Figure~\ref{fig:synth_pval_contamination}, where our post-shift conditional CTM p-values remain high, in stark contrast to the decaying evidence of the standard CTM p-values.

While our approach eliminates test-time contamination, the use of a fixed reference $D_0$ introduces a subtle but crucial challenge: test points are repeatedly compared to the same finite reference set, and these repeated comparisons induce dependencies that break the i.i.d. structure assumed under the null \citep{bates2023testing}. Indeed, a naive implementation of CTM that places bets using only the fixed reference scores in $D_0$ (rather than on a growing, random reference pool $D_t$ as in standard CTM constructions) leads to inflation of the type-I error rate. This is illustrated in the left panel of Figure~\ref{fig:motivation_combined} (see the `Invalid CTM' curve), where the lack of error control is more severe when the reference dataset $D_0$ is relatively small.

By contrast, if we had access to the true null distribution $P$, we could directly form an anytime-valid test. Specifically, applying the probability integral transform (PIT) using the true CDF of $P$ yields p-values that are guaranteed to be i.i.d. uniform under the null, thereby circumventing the dependency issues caused by the finite reference set and enabling standard betting against uniformity. In reality, however, the true CDF is unknown and should be estimated using $D_0$. Consequently, we must explicitly address the estimation uncertainty that arises from this finite sample approximation. Technically, to quantify how atypical $X_t$ is relative to $D_0$, we compute a conformal $p$-value using the empirical cumulative distribution function (ECDF) of $D_0$ that estimates the null distribution $P$. We then bound the estimation error of the ECDF by deriving upper and lower confidence bands (e.g., via the Dvoretzky–Kiefer–Wolfowitz inequality), ensuring that the population CDF lies within this band with high probability. This, in turn, allows us to construct a betting martingale that is robust to estimation errors of the null distribution, rigorously controlling the type-I error for all time steps at level $\alpha$ with high probability (over the draw of $D_0$). As shown in Figure~\ref{fig:motivation_combined}, our reference-conditional CTM approach controls the type-I error while demonstrating higher power with faster detection.

Beyond the ability to improve power empirically, conditioning on $D_0$ allows us to theoretically analyze the power of the proposed test. We derive conditions under which our conditional CTM (i) achieves asymptotic power one, and (ii) has an upper bound on the expected time to reject the null when it is false. To the best of our knowledge, such guarantees are not available in the existing CTM literature. Finally, our experiments demonstrate that when the reference set is of reasonable size (e.g., exceeding $500$ samples), the proposed method achieves higher power and faster detection times than standard CTM and other competitive baselines.

We summarize our contributions as follows:
\begin{enumerate}
\item We introduce the \emph{conditional conformal test martingale}, a sequential test for distribution shift detection that avoids test-time contamination by design. A key feature is that shift evidence is computed using a fixed reference dataset $D_0$, rather than a growing dataset that incorporates test points as in standard CTM methods.


\item We provide finite-sample, anytime type-I error guarantees. We do so by constructing a novel betting function that accounts for estimation errors of the null distribution from the fixed reference $D_0$

\item We derive power guarantees for our conditional CTM test, including conditions for asymptotic power one and bounds on the expected detection delay.

\item We validate our theory empirically, showing improved power and faster detection compared with standard CTM and other baselines, for both change-point detection and drift detection. \footnote{A software package that implements our method is available at \url{https://github.com/shaersh/cctm}.}
\end{enumerate}

\section{Background and Related Work}
\label{sec:bg}

\subsection{Testing by Betting with Test Martingales}
\label{sec:betting}
The testing by betting paradigm \citep{shafer2021testing, ramdas2023game} conceptualizes sequential hypothesis testing through a non-negative wealth process, often termed a testing martingale. This process is initialized and dynamically updated with each observation.
More formally, the non-negative testing martingale $S_t$ is initiated at $S_0=1$ and updated by:
\begin{equation}
    \label{eq:bet_form}
    S_t = S_{t-1} \cdot b_t(X_t).
\end{equation}
Here, $b_t(\cdot) \in [0,\infty)$ is an $\mathcal{F}_{t-1}$-measurable non-negative betting function satisfying the constraint $\mathbb{E}_{H_0}[b_t(X_t) \mid \mathcal{F}_{t-1}] \leq 1$, where $\mathcal{F}_{t-1}$ is a filtration denoting all available information before observing $X_t$; for example, defining $\mathcal{F}_{t-1} := \sigma(X_1,...,X_{t-1})$ would contain the information of all the observed scores, $(X_1,...,X_{t-1})$.
This constraint ensures $S_t$ is a non-negative super-martingale under $H_0$:
\begin{align}
    \label{eq:martingale_null}
    \mathbb{E}_{H_0}[S_t \mid \mathcal{F}_{t-1}] &= \mathbb{E}_{H_0}[S_{t-1} \cdot b_t(X_t) \mid \mathcal{F}_{t-1}] \\ &= S_{t-1} \cdot \mathbb{E}_{H_0}[b_t(X_t) \mid \mathcal{F}_{t-1}] \leq S_{t-1}, \notag
\end{align}
which, in turn, provides the foundation for type-I error control. By Ville's inequality \citep{ville1939etude}, 
for any non-negative super-martingale $(S_t)_{t \geq 0}$ with $S_0=1$, we have:
$$
\mathbb{P}_{H_0}\left[\exists t \geq 1 : S_t \geq 1/\alpha\right] \leq \alpha,
$$
for any choice of a test level $\alpha \in (0,1)$. This inequality guarantees that the probability of the testing martingale ever crossing the rejection threshold $1/\alpha$ is at most $\alpha$, which establishes time-uniform type-I error control. For a brief survey of related work on testing by betting, see Appendix~\ref{appdx:betting_related}.


\subsection{Conformal Test Martingales}
\label{sec:ctm}

CTM \citep{vovk2003testing,vovk2021testing, vovk2021retrain} is a non-parametric method originally designed to test the null hypothesis of exchangeability, a weaker notion than i.i.d.
The core idea is to evaluate the extent to which the new score $X_t$ obtained at time $t$ is ``atypical'' compared to past observations.
More formally, 
at each time $t$, the method considers the set of all existing data $D_t = D_0 \cup \{X_1, \dots, X_t\}$, estimates the ECDF $\hat{F}_t$ using $D_t$, and
computes a randomized p-value $p_t$ by applying $\hat F_t$ to the newest point $X_t$:\footnote{Standard CTMs often define p-values based on the observed fraction of larger a non-conformity score $\mathbf{1}\{X_i > X_t\}$. In our case, the use of `$>$' or `$<$' does not matter, as we run a two-sided test for uniformity; both small and large $p_t$ provide evidence against the null.}
\begin{equation}
\label{eq:cdf_ctm}
    \resizebox{\columnwidth}{!}{$
    p_t := \hat{F}_t(X_t) = \frac{1}{t+n} \sum_{X_i \in D_t} \left( \mathbbm{1}\{X_i < X_t\} + U_t \cdot \mathbbm{1}\{X_i = X_t\} \right),
    $
    }
\end{equation}
    where $n=|D_0|$ is the number of samples in $D_0$, and $U_t$ is a uniform random variable drawn independently from $\mathcal{U}[0,1]$. 
    Note that \eqref{eq:cdf_ctm} can roughly be interpreted as the fraction of previous scores at least as small as the current score, with more extreme values of $p_t$ conveying greater evidence.

The key property of this construction is that, under the null hypothesis of exchangeability, the resulting sequence $p_1, p_2, \dots$ is i.i.d. uniform on $[0,1]$ \citep{vovk2003testing}. 
This allows for the construction of a test martingale, by betting against $p_t \overset{\text{i.i.d.}}{\sim} U[0,1]$ for all $t \geq 1$. There exist various test martingales for uniformity \citep{fedorova2012plug, vovk2021retrain}, where a common approach---which we use in this paper---is to test whether the sequence $p_1, p_2, \dots$ has a mean of $0.5$. In this case, the non-negative betting function is of the form:
\begin{equation}
    \label{eq:betting_unif}
    b_t(p_t) = 1 + \eta_t \cdot (p_t - 0.5),
\end{equation}
where $\eta_t \in [-2,2]$ is a predictable betting parameter. Observe that if the sign of $\eta_t$ and $(p_t-0.5)$ is the same, then $b_t(p_t) > 1$, which results in a successful bet that increases the wealth~\eqref{eq:bet_form}. Therefore, the goal of $\eta_t$ is to predict the direction of the deviation of $p_t$ from $0.5$ based on the previous p-values $p_1,...,p_{t-1}$, where its magnitude affects how aggressive the bet is. \sh{We note that the choice of the betting function in~\eqref{eq:betting_unif} is not arbitrary; recent works show that this form is optimal for betting for mean estimation of bounded random variables under proper constraints \cite{clerico2024optimal,clerico2025optimality}.}

Importantly, this design of the standard CTM, where $\hat F_t$ is continuously updated with each new sample $X_t$, exposes the procedure to ``test-time contamination,'' as discussed in Section~\ref{sec:intro}. This contamination can slow down detection and compromise the test's power; accordingly, it seems appealing to instead only allow $\hat{F}_t$ to be estimated based on a fixed holdout set $D_0$, i.e., $\hat{F}_0$. However---and crucially---naive use of only a fixed holdout set violates the independence between subsequent $p$-values 
\citep{bates2023testing} and breaks the validity of the martingale. 
Put differently, when conditioning on the finite $D_0$ without further adjustment, the estimation error in $\hat{F}_0\approx F$ (where $F$ is the true CDF) generally creates a systematic bias in the distribution of $p_t$ that manifests in dependence between 
subsequent $p$-values, which causes the ``wealth process'' to no longer be a ``fair game'' and thereby breaks the (super)martingale condition.


In this work, we close this gap by utilizing the fixed holdout set $D_0$ to model $\hat{F}_0$ while rigorously accounting for the dependence between $p_t=\hat{F}_0(X_t)$ arising from the imperfect estimation of the true CDF ${F}$.

\subsection{Additional Related Work}
\label{sec:related}

Our work is related to the work by \citet{bates2023testing},
which offers a framework to test for outliers using conformal p-values. Specifically, \citet{bates2023testing} showed that when conformal p-values are evaluated on a fixed holdout reference set they can yield invalid type-I error control when testing for a global null: whether there is one or more outliers in a test set. They then proposed a correction by using concentration inequalities, such as the Dvoretzky–Kiefer–Wolfowitz (DKW) inequality, to construct more conservative p-values that account for the uncertainty in the ECDF estimate. However, \emph{their work focuses on offline, fixed-time error control, and our goal is to achieve anytime-valid type-I error control in an online setting}. We utilize the same core principle of bounding the ECDF estimation error (e.g., via DKW confidence bands), but we develop a specialized betting function that is robust to these estimation errors. This construction allows us to transition from an offline to an online setting while guaranteeing error control simultaneously over all time steps $t$.

Our method also contributes to the general literature on online, anytime-valid methods for testing exchangeability and the i.i.d. assumption. Beyond the standard CTM framework \citep{vovk2021testing, vovk2003testing}, related methodologies include pairwise betting \citep{saha2024testing} and an approach based on universal inference \citep{ramdas2022testing, wasserman2020universal}, both of which develop anytime-valid tests with theoretical consistency guarantees. The universal inference approach \citep{ramdas2022testing} is primarily designed for binary sequences, making it less applicable to the general univariate score $X \in \mathbb{R}$ setting considered here. The pairwise betting approach \citep{saha2024testing} tests for exchangeability by betting on the ordering of sample pairs. Crucially, this approach uses all observed data to fit the distributional model underpinning their bettor. Consequently, when a shift occurs after a long sequence of in-distribution observations, the learned model parameters are dominated by the in-distribution data, resulting in a delayed detection of the distribution shift.
Most recently, \citet{fischer2025sequential} also develop a sequential Monte Carlo test for exchangeability with anytime-valid guarantees, but this test is for a particular alternative hypothesis; in contrast, our methods allow for simultaneously testing against \textit{any} alternative hypothesis, to detect arbitrary distribution shifts.


\section{Proposed Method}
\label{sec:proposed}

In this section, we propose the conditional CTM method: an online test for $H_0$~\eqref{eq:null} that is valid conditionally on the finite sample reference set $D_0$. 
All proofs are in Appendix~\ref{appdx:proofs}.

\subsection{Forming the Test}
\label{sec:martingale}

To set the stage, we first consider an ideal setting where the true CDF of $P$, denoted as $F$, is known. Under the null hypothesis $H_0$, the true CDF values $F(X_t)$ are \emph{i.i.d. uniform} on $[0,1]$ by the probability integral transform. 
This implies that $\mathbb{E}_{H_0}[F(X_t)] = 0.5$; a property that the standard CTM utilizes, see~\eqref{eq:betting_unif}. A distribution shift (under an alternative $H_1$) can lead to a deviation from uniformity, resulting in $\mathbb{E}_{H_1}[F(X_t)] \neq 0.5$. Consequently, in this ideal case, we can test for $H_0$ using the same betting function from~\eqref{eq:betting_unif}.

In practice, the true CDF $F$ is unknown, and we must rely on its empirical estimate evaluated on the finite reference set $D_0$. We define the ECDF as $\hat{F}_0(x):=\frac{1}{n}\sum_{i=1}^n \mathbbm{1}\{X_i^0 \leq x\}$. Since the transformed values $\hat{F}_0(X_t)$ exhibit deviations from the ideal ${F}(X_t)$ due to estimation errors, may not be i.i.d. uniform on $[0,1]$ under $H_0$. Naively substituting $\hat{F}_0$ for $F$ in the betting function can therefore violate the martingale property under the null~\eqref{eq:martingale_null}, leading to an uncontrolled type-I error. Indeed, this is precisely how the invalid CTM method from Figure~\ref{fig:motivation_combined} is implemented; observe that this method fails to control the type-I error.\footnote{The invalid CTM is implemented by computing the \emph{randomized} conformal p-values (using only $D_0$) to ensure the variables are uniform. This emphasizes that the test is invalid as it violates the i.i.d. assumption under the null~\citep{bates2023testing}.}

To rigorously account for the estimation error of $\hat{F}_0$, we construct uniform confidence bounds for $F$ at a given level $\delta$ using $D_0$, such that
\begin{equation}
\label{eq:CI}
\mathbb{P}_{D_0} \bigl(|\hat{F}_0(x) - F(x)| \leq \epsilon(x), \text{ for all } x \in \mathbb{R} \bigr) \geq 1-\delta.
\end{equation}
Above, $\epsilon(x) \in [0,1]$ represents the half-width of the confidence band around $\hat{F}_0(x)$, which may vary for different values of $x$.\footnote{We assume the bounds are symmetric around $\hat{F}_0(x)$ for simplicity. The extension to asymmetric bounds is straightforward.}
A standard way to construct bounds satisfying~\eqref{eq:CI} is by using the DKW inequality \citep{dvoretzky1956asymptotic}. This provides a constant-width bound $\epsilon(x) = \epsilon_n$, where $\epsilon_n = \sqrt{\frac{1}{2n}\log(\frac{2}{\delta})}$ depends only on the reference set size $n$ and the confidence level $\delta$. There exist other methods that provide uniform bounds with varying width, i.e., of non-constant $\epsilon(x)$ \citep{owen1995nonparametric}. Consequently, our remaining analysis assumes the general form in \eqref{eq:CI}, which allows for bounds that are not necessarily constant.

We now introduce a novel betting function designed to be valid conditional on $D_0$ by explicitly incorporating the estimation error of $\hat{F}_0$.
Let $\hat{p}_t=\hat{F}_0(X_t)$ and define $\mathcal{F}_{t} = \sigma(\hat{p}_1,...,\hat{p}_{t})$ as the filtration generated by $\hat{p}_1,...,\hat{p}_{t}$.
We define the wealth process $S_t = S_{t-1}\cdot b_t(\hat{p}_t)$ (with $S_0=1$) that is adapted to $\mathcal{F}_{t}$, where our proposed betting function is given by
\begin{equation}
    \label{eq:non_diff_b}
    b_t(\hat{p}_t) = 1 + \eta_t \cdot (\hat{p}_t - 0.5) - |\eta_t| \cdot \epsilon(\hat{p}_t).
\end{equation}
Here, the term $1 + \eta_t (\hat{p}_t - 0.5)$ mirrors the standard CTM betting formula (Section~\ref{sec:ctm}), while the correction term in the right hand side $- |\eta_t| \cdot \epsilon(\hat{p}_t)$ removes gains that might arise solely from estimation errors. Indeed, in the ideal case where $\epsilon(\hat{p}_t)=0$ (i.e., $\hat F_0=F$), the bet $b_t(\hat{p}_t) > 1$ as long as $\eta_t$ and $(\hat{p}_t - 0.5)$ have the same sign, resulting in wealth growth since $S_{t}=S_{t-1}\cdot b_t(\hat{p}_t)$. When accounting for estimation error ($\epsilon(\hat{p}_t)>0$), however, systematic growth in wealth occurs only if the deviation $|\hat{p}_t - 0.5|$ exceeds $\epsilon(\hat{p}_t)$; and this event is unlikely to happen under the null. In other words, evidence for a shift is only accumulated when the observed deviation is sufficiently large to overcome the uncertainty introduced by finite-sample estimation of $F$. 

The betting parameter $\eta_t$ in~\eqref{eq:non_diff_b} is predictable, i.e., it is measurable with respect to the history $\mathcal{F}_{t-1} = \sigma(\hat{p}_1, \dots, \hat{p}_{t-1})$ {and must be set without looking at $\hat p_t$}. This allows us to learn how to predict $\eta_t$ given past data and $D_0$, with the goal of maximizing the accumulated wealth.
To effectively learn the betting variable $\eta_t$ in a our online setting, we employ the online newton step (ONS) algorithm \citep{hazan2007logarithmic, cutkosky2018black}, as we detail later in Section~\ref{sec:online_learning_ons}. ONS is a second-order gradient-based method which requires the betting function to be twice differentiable with respect to $\eta_t$. Since our proposed betting~\eqref{eq:non_diff_b} includes the term $|\eta_t|$, which is non-smooth at zero, we introduce a smoothed approximation $|\eta_t| \approx \sqrt{\eta_t^2 + k^2}$, where $0<k\ll1$ is a small constant.\footnote{We note that $k>0$ can be as small as desired; we use this property in Section~\ref{sec:power_analysis}, where we analyze the power of the proposed test. We set $k=10^{-6}$ in our experiments.} This yields the final smoothed betting function:
\begin{align}
    \label{eq:betting}
    b_t(\hat{p}_t) &= 1+g_t(\hat{p}_t,\eta_t) \\
    &:= 1 + \mathcal{C} \cdot (\eta_t \cdot (\hat{p}_t - 0.5) -\sqrt{\eta_t^2+k^2}\cdot \epsilon(\hat{p}_t)), \notag
\end{align}
where $\eta_t \in [-1,1]$ and $\mathcal{C} := \C >0$ is a scaling constant that ensures $b_t(\hat{p}_t)$ lies in $[0,2]$.\footnote{In practice, for ONS we restrict $\eta_t \in [-0.5,0.5]$ to ensure that $g_t(\hat{p}_t,\eta_t) \in [-0.5,0.5]$ for all $\hat{p}_t,\eta_t$ and $\epsilon(\hat{p}_t)$.} \

Lastly, 
to perform the test at a given level $\alpha \in (0,1)$, we monitor $S_t$ for $t=1,2,\ldots$, and we reject the null hypothesis the first time it exceeds $1/\alpha$. The full test procedure is outline in Algorithm~\ref{alg:CCTM}.

In the following theorem, we show that the type-I error of the proposed test is bounded by $\alpha$, with probability $1-\delta$ over the randomness of $D_0$. \sh{That is, as long as the true CDF $F$ is within the $D_0$-dependent confidence bounds~\eqref{eq:CI}, which occurs with a probability of at least $1-\delta$, the probability of ever falsely rejecting $H_0$ is bounded by $\alpha$.}
Importantly, the guarantee holds in finite samples and for any null  $P$.
 \begin{theorem}
     \label{thm:err_control}
     Let $\mathcal{F}_{t-1} := \sigma(\hat{p}_1,...,\hat{p}_{t-1})$ be the filtration generated by $\hat{p}_1,...,\hat{p}_{t-1}$.
     Given an ECDF $\hat{F}_0$ of the null distribution $P$, estimated with the reference set $D_0$, and corresponding confidence bounds $\epsilon(\cdot)$ that satisfy~\eqref{eq:CI} for a chosen $\delta \in (0,1)$, then for any $\alpha \in (0,1)$:
     $$
     \mathbb{P}_{D_0}\bigl( \mathbb{P}_{H_0}\left( \exists t \geq 1 : S_t \geq 1/\alpha \mid \mathcal{F}_{t-1}, D_0 \right) \leq \alpha \bigr) \geq 1-\delta.
     $$
 \end{theorem}
Notably, although the result above is formulated using the univariate scores $X_t$, its validity extends directly to the raw data stream $Z_1,Z_2,...$. Specifically, as long as the score function $\hat{s}$ is fixed (e.g., the confidence of a fixed classifier applied to $Z_t$), Theorem~\ref{thm:err_control} ensures that our test rigorously controls the probability of a false alarm for a distribution shift in the raw data stream. 

\begin{algorithm}[t]
\caption{Conditional Conformal Test Martingale}\label{alg:CCTM}
\begin{algorithmic}
    \STATE \textbf{Input:} 
    \STATE \ \  $D_0=\{X^0_i\}_{i=1}^n$: fixed reference set
    \STATE \ \  $X_1,...,X_t,...$: stream of test samples
    \STATE \ \  $\alpha$: test level
    \STATE \ \ $\delta$: level of the CDF confidence bounds
    \STATE \ \ $k>0$: smoothing parameter
    \vspace{0.5em} 
    \STATE \textbf{Algorithm Steps:}
    \STATE Estimate $\hat{F}_0(x):=\frac{1}{n}\sum_{i=1}^n \mathbbm{1}\{X_i^0 \leq x\}$
    \STATE Initialize $S_0 \gets 1$, $\eta_1 \gets 0$
    \STATE \textbf{for} {$t=1,...$} \textbf{do}
    \STATE \quad $\hat{p}_t \gets \hat{F}_0(X_t)$
    \STATE \quad Evaluate $\epsilon(\hat{p}_t)$ with $D_0$ and $\delta$ (e.g., with DKW)
    \STATE \quad $S_{t} \gets S_{t-1} \cdot b_t(\hat{p}_t)$ where $b_t(\hat{p}_t)$ defined in \eqref{eq:betting}
    \STATE \quad \textbf{if} {$S_t \geq 1/\alpha$} \textbf{then}
    \STATE \quad \quad reject $H_0$
    \STATE \quad Update $\eta_{t+1}$ by applying Algorithm~\ref{alg:eta_ons_smooth} 
\end{algorithmic}
\end{algorithm}

\subsection{Online Learning of the Betting Parameter $\eta_t$}
\label{sec:online_learning_ons}

\begin{algorithm}[t]
\caption{Online Newton Step (ONS) to learn $\eta_t$ online}\label{alg:eta_ons_smooth}
\begin{algorithmic}
    \STATE \textbf{Input:} 
    \STATE \ \ $\hat{p}_t = \hat{F}_0(X_t) \in [0,1]$: conformal p-value 
    \STATE \ \  $\eta_t$: last betting parameter's value where $\eta_1=0$
    \STATE \ \ $a_{t-1}$: previous ONS parameter value where $a_0=1$
    \STATE \ \ $\epsilon(\hat{p}_t)$: confidence bound parameter as in \eqref{eq:CI}
    \vspace{0.5em} 
    \STATE \textbf{Algorithm Steps:}
    \STATE $g_t \gets \mathcal{C} \cdot (\eta_t \cdot (\hat{p}_t - 0.5) -\sqrt{\eta_t^2+k^2}\cdot \epsilon(\hat{p}_t))$
    \STATE $g'_t \gets \mathcal{C} \cdot (\hat{p}_t - 0.5 -\frac{\eta_t}{\sqrt{\eta_t^2+k^2}}\cdot \epsilon(\hat{p}_t))$
    \STATE $z_t \gets {g'_t}/{(1+g_t)}$
    \STATE $a_t \gets a_{t-1} + z_t^2$
    \STATE $\eta_{t+1} \gets \max(-\frac{1}{2},\min(\eta_t + 4 \cdot \frac{z_t}{a_t},\frac{1}{2}))$
    \vspace{0.5em}
    \STATE \textbf{Output:} $\eta_{t+1}$
\end{algorithmic}
\end{algorithm}

We now turn to present an online learning approach to predict the betting variable $\eta_t$ from past observations. Our goal is to maximize the wealth by minimizing the negative log of the martingale up to timestep $T$:
\begin{align}
    \label{eq:log_loss}
    &-\ln(S_T) = -\sum_{t=1}^T \ln(1+g_t(\hat{p}_t,\eta_t)),
\end{align}
where $g_t(\hat{p}_t,\eta_t)$ is as defined in~\eqref{eq:betting}.
Therefore, we denote the loss at time $t$ as the negative log of the betting function parameterized by $\eta$:
\begin{equation}
    \label{eq:loss_ons}
    \ell_t(\eta) = -\ln\left(1+\mathcal{C} \cdot (\eta(\hat{p}_t - 0.5) - \sqrt{\eta^2 + k^2}\cdot \epsilon(\hat{p}_t))\right).
\end{equation}
Notably, the above loss is twice differentiable with respect to $\eta$, allowing us to adapt ONS to this setting and learn how to predict $\eta_t$ from past samples.
The complete algorithm is presented in Algorithm~\ref{alg:eta_ons_smooth}.

While ONS is widely used in the testing by betting literature \citep{cutkosky2018black, shekhar2023nonparametric, chugg2023auditing, dai2025individual}, we note that its application here differs from standard settings. In typical martingale-based sequential tests, the betting parameter $\eta$ usually multiplies a single outcome, i.e., $\ln(1+\eta_t h_t(\hat{p}_t))$ where $h_t(\hat{p}_t)$ does not depend on $\eta_t$. In contrast, in our formulation~\eqref{eq:loss_ons}, $g_t(\hat{p}_t,\eta_t)$ cannot be separated into a simple product of $\eta_t$ and a term independent of it. This, in turn, requires a different analysis of the learning process. The complete analysis is detailed in Appendix~\ref{appdx:online},
where we establish that the cumulative negative log wealth up to time $T$~\eqref{eq:log_loss} is bounded by that of the best fixed strategy in hindsight, plus a term logarithmic in time.
This bound is crucial for the power analysis presented next, as it allows us to set the conditions for asymptotic power one and to derive the expected time to reject the null. 


\subsection{Power Analysis}
\label{sec:power_analysis}

To analyze the power of the proposed test, we consider a fixed alternative distribution:
\begin{equation}
    \label{eq:H1}
    H_1: X_t \overset{\text{i.i.d.}}{\sim} Q, \ \ \forall \ t\geq 1,
\end{equation}
where $Q \neq P$. Note that our testing framework is valid against composite alternatives, but we focus here on a fixed alternative to provide a formal power analysis.

We characterize the difficulty of the detection task by defining a quantity that measures the extent to which the alternative distribution $Q$ deviates from $P$. Formally, we refer to it as the effective signal strength $\Delta_k$ with respect to the estimation error, given by:
\begin{equation}
    \label{eq:delta}
    \Delta_k:= |\mathbb{E}_{H_1}[\hat{p}_t - 0.5 \mid D_0]| - \sqrt{1+k^2} \Bar{\epsilon},
\end{equation}
where $\Bar{\epsilon} := \mathbb{E}_{H_1}[\epsilon(\hat{p})]$, with the expectation taken under the alternative $H_1$. Intuitively, $\Delta_k$ can be viewed as the margin by which the signal used for betting $|\mathbb{E}_{H_1}[\hat{p}_t - 0.5 \mid D_0]|$ exceeds the CDF estimation error term $\sqrt{1+k^2}\Bar{\epsilon}$. Notably, as $|D_0| \rightarrow \infty $ we have that $\Bar{\epsilon} \rightarrow 0$. In turn, $\Delta_k$ increases with the size of the reference set $D_0$.

Having defined the alternative $H_1$ and the effective signal strength $\Delta_k$, we now state the first result.
\begin{lemma}
    \label{lemma:expected_bet}
     Under a fixed alternative $H_1$, i.e., $X_t \sim Q$ where $Q \neq P$ for all $t \geq 1$, and given the reference set $D_0$. For a small enough $k$, if $\Delta_k>0$ then $\max_{\eta \in [-\frac{1}{2},\frac{1}{2}]} \mathbb{E}_{H_1}{\ln(1 + g(\hat{p},\eta))} > 0$.
\end{lemma}
The exact upper bound for $k$ is provided in the proof in Appendix~\ref{appdx:proof_lemma_expected_bet}. In simple terms, Lemma~\ref{lemma:expected_bet} asserts that if the effective signal strength $\Delta_k$ is positive, there exists a fixed betting variable $\eta$ that yields a positive expected gain. This positive gain is essential to drive the test martingale to grow, ensuring that the null hypothesis will eventually be rejected, as stated below.

\begin{theorem}
     \label{prop:power_one}
     Under the same assumptions as in Lemma~\ref{lemma:expected_bet},
     the test presented in Algorithm~\ref{alg:CCTM} has asymptotic power one. Further, for $k \rightarrow 0$, the expected stopping time $\tau := \inf\{t \geq 1:S_t \geq 1/\alpha\}$ for level $\alpha$ test is bounded by 
     $$
     \mathbb{E}_{H_1}{\tau} = \mathcal{O}\left( \frac{1}{\Delta_0^2} \ln \left(\frac{1}{\alpha \Delta_0}\right) + \frac{1}{\Delta_0^2} \right).$$
\end{theorem}
The above result implies that our proposed test is guaranteed to reject $H_0$ with probability one if $Q$ sufficiently deviates from $P$. Furthermore, the expected stopping time scales with the order of $\Delta_0^{-2}$, indicating that detection time reduces as the magnitude of the distribution shift increases. Crucially, to the best of our knowledge, this is the first result establishing asymptotic power one and providing bounds on the expected stopping time for CTM-based procedures.

\sh{
\begin{remark_non} 
\textit{We note that the $\mathcal{O}(\Delta_0^{-2} \ln(1/\alpha\Delta_0))$ stopping time bound we achieve aligns with results established in recent betting literature (e.g., \citet{chugg2023auditing,gauthier2026betting}). However, establishing this bound in our setting introduces a unique technical challenge. Standard sequential tests typically use betting functions that can be separated into the betting parameter $\eta_t$ and the observed $\hat p_t$ (i.e., $b_t(\hat p _t)=1+\eta_t \cdot  g_t(\hat p_t)$), whereas our betting function~\eqref{eq:betting} cannot. Consequently, achieving this stopping time bound requires substantial new theoretical derivations, such as deriving tailored regret bounds for the ONS algorithm (Appendix~\ref{appdx:online}), and subsequently proving the asymptotic power and bounded expected stopping time (Theorem~\ref{prop:power_one}).
}
\end{remark_non} 
}

\subsection{Practical Consideration: Mitigating Wealth Decay}
\label{sec:practical}


Under $H_0$, the best choice for the betting parameter is $\eta_t = 0$, since systematic wealth growth is impossible; any bet with $\eta_t \neq 0$ risks reducing wealth. In practice, however, due to stochasticity, the online learning scheme produces $\eta_t$ values that fluctuate around zero. This can lead to a decay in the wealth process, which is problematic: if a shift occurs after a long period, the martingale must first recover the lost wealth before it can cross the rejection threshold, thereby increasing the detection delay. This limitation is observed both for our method and for the standard CTM, and is well known in the literature \citep{volkhonskiy2017inductive}. In our case, this decay can be further amplified, since there is a wider range of $\eta_t$ values for which the bet is smaller than 1.

To mitigate this, we clip $\eta_t$ to zero whenever it is small. Concretely, at each step $t$, after ONS proposes $\eta_t$, we set $\eta_t = 0$ if $|\eta_t| < C$, where $C \ge 0$ is a clipping parameter ($C=0$ means no clipping; in our experiments $C=0.1$). This ensures that when the data stream appears to follow the null, no bet is placed and no wealth is lost. 
We provide an ablation study on the parameter $C$ in Appendix~\ref{appdx:clipping}, which shows that the clipping strategy effectively improves the power of both our method and the standard CTM.
\vspace{-1pt}
\section{Experiments}
\label{sec:exp}

In this section, we test the performance of our proposed conditional CTM, demonstrating its advantages in terms of power and detection speed,\footnote{Recall that we empirically validate the type-I error control of the conditional CTM in Figure~\ref{fig:motivation_combined}.} compared to the standard CTM. We also provide benchmarks against the state-of-the-art pairwise betting approach in more specialized synthetic experiments, as pairwise betting provides a betting function specifically tailored for Gaussian AR(1) processes, establishing a strong benchmark for comparison. 
We consider two use-cases that are common in the literature: (i) change-point detection, and (ii) drift detection. 


\subsection{Synthetic Experiments}
\label{sec:synth}



\begin{figure}[t!]
    \centering
    \includegraphics[width=\columnwidth]{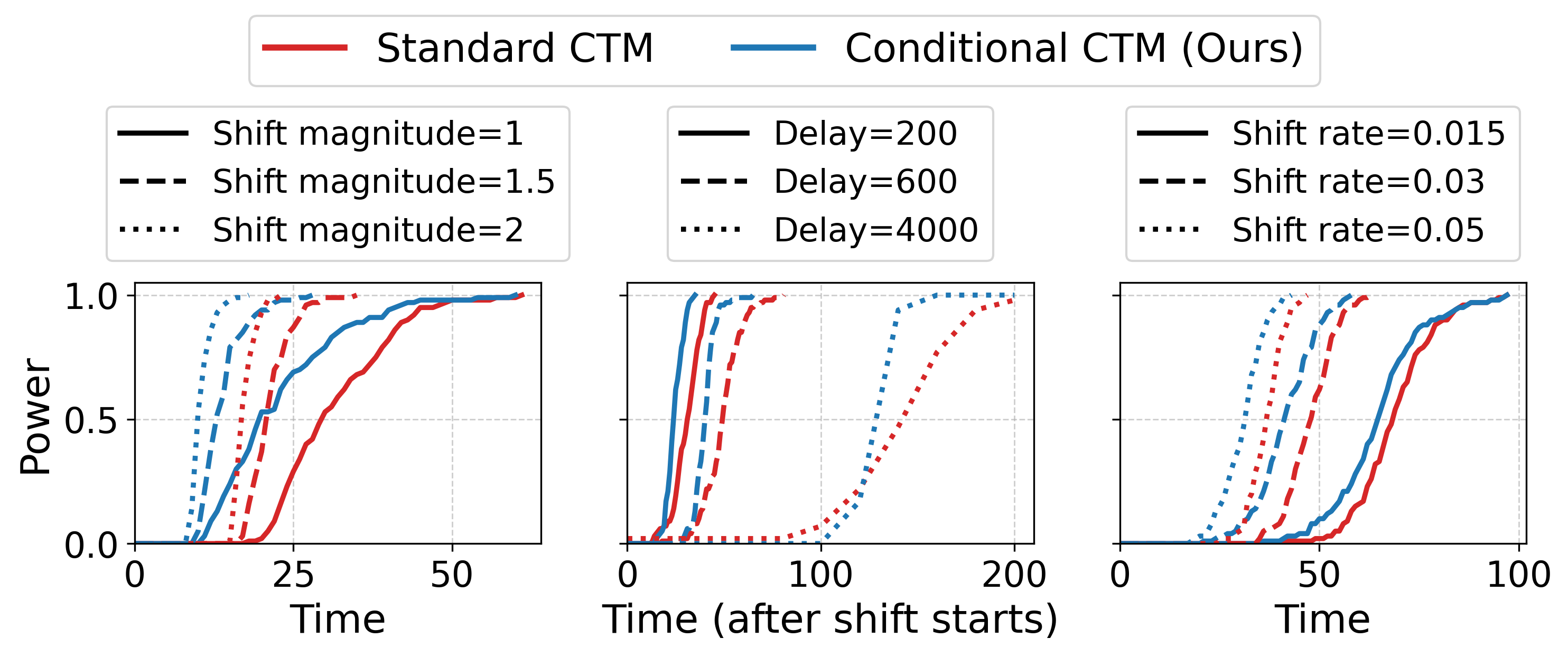}
    \caption{Empirical power of the standard CTM and our conditional CTM in three scenarios. In each scenario, the empirical power is evaluated over 100 trials. \textbf{Left:} an immediate change-point case with different shift magnitudes. \textbf{Middle:} a delayed change-point scenario with different shift delays. \textbf{Right:} gradual shifts of different rates.}
    \label{fig:three_panes}
\end{figure}






\textbf{Experimental setup.}
Throughout the synthetic experiments, we consider the null distribution $P$ to be $\mathcal{N}(0,1)$, and generate the reference set $D_0$ by sampling $n=1000$ independent samples from $P$. We set the target type-I error rate to $\alpha = 0.05$, and construct the CDF confidence bound $\epsilon$ using DKW with a confidence level of $\delta = 0.1$. The 
smoothing parameter is set to $k=10^{-6}$. 
We benchmark our method against the standard CTM from Section~\ref{sec:ctm}. To ensure a fair comparison, we implement the CTM method with ONS (as in our method) to predict the betting parameter; we use the same clipping parameter ${C}=0.1$.

\textbf{Immediate change-point.}
We start by testing the power of our method when the change-point happens immediately, at timestep $t=1$. To this end, we simulate a data stream, sampled from an alternative (shifted) distribution $X_t \sim \mathcal{N}(d,1)$ for all $t \geq 1$, where $d$ denotes the shift magnitude relative to $P$. Figure~\ref{fig:three_panes} (left) presents the empirical power over 100 trials across varying magnitudes of $d$. As can be seen, our conditional CTM consistently achieves higher power and reduced detection delays compared to the standard CTM for all values of $d$. Notably, this advantage is observed despite both methods have the same size of null data $D_0$ that is observed before the shift occurs. 

\textbf{Delayed change-point.}
We now generate data such that the change-point happens at timestep $t_0 \gg 1$, i.e., $X_t \sim P$ for $t < t_0$, and shifts to $X_t \sim \mathcal{N}(2,1)$ for $t \geq t_0$. Figure~\ref{fig:three_panes} (middle) depicts the empirical power as a function of the time steps elapsed after the shift ($t - t_0$) for delays $t_0 \in \{200, 600, 4000\}$.
\sh{As portrayed, for moderate delays ($t_0 \in \{200, 600\}$), our conditional CTM outperforms the standard CTM. Interestingly, when the shift occurs late ($t_0 = 4000$), the standard CTM initially reacts slightly faster immediately following the change-point, as it benefits from the accumulated in-distribution samples. However, as post-shift observations begin to contaminate the standard CTM's growing reference set $D_t$, our conditional CTM overtakes the standard method, ultimately achieving higher power earlier.}


\textbf{Gradual shifts.}
To demonstrate the adaptivity of the test in dynamic environments where severity of the shift increases over time, we simulate data $X_t$ with time-dependent shifted mean, i.e., $X_t \sim \mathcal{N}(\lambda \cdot t,1)$. We vary the slope $\lambda$ to simulate drifts of different speeds. As illustrated in Figure~\ref{fig:three_panes} (right),  our method achieves higher power and faster detection compared to the standard CTM. 


\textbf{Comparison to pairwise betting.}
To provide a comparison to the pairwise betting approach \citep{saha2024testing}, we consider $X_t$ to be a stationary Gaussian AR(1) process, i.e., $X_{t+1} \mid X_t = \mathcal N (aX_t, \sigma^2)$. We set $a=0.7$ and $\sigma^2=1$, and apply the pairwise test as detailed in Section 2.2 of \citep{saha2024testing}, where $a$ and $\sigma^2$ are estimated online from the test data.
Figure~\ref{fig:pairwise}, presented in Appendix~\ref{appdx:synth_figs}, depicts the power of our method, standard CTM and the pairwise test. As observed, our conditional CTM achieves the fastest detection, the standard CTM follows after, and the pairwise method exhibits a significant delay. This delay can be explained by the need to estimate the AR parameters online and its limitation of betting on pairs, meaning it accumulates evidence half as fast as CTM and our test.
\vspace{-3.5pt}
\subsection{Experiments on ImageNet-C}
\label{sec:imgnt}

We evaluate our method on ImageNet-C \citep{hendrycks_2018_2235448}, a standard benchmark designed to assess model robustness against realistic distribution shifts induced by common image corruptions. 
ImageNet-C is derived from the ImageNet validation set by applying 15 corruption types at five severity levels; in this work, unless stated otherwise, we use the highest severity level. Following our testing framework, we assume our reference set $D_0$ of size $n$ is drawn from the clean (regular ImageNet) domain, and then monitor a sequential test stream $(X_t)_{t\ge 1}$ drawn from a corrupted domain. 
The score we use is the entropy of the estimated class probabilities obtained by a fixed pre-trained ViT-Base classifier. That is, $D_0$ consists of the entropies computed on clean images, and for each type of corruption,  $(X_t)$ consists of the entropies computed on corrupted images. More implementation details are in Appendix~\ref{sec:implementation}.

    



\begin{figure}[t]
\centering

\begin{subfigure}[b]{0.26\textwidth}
  \centering
  \includegraphics[width=\textwidth]{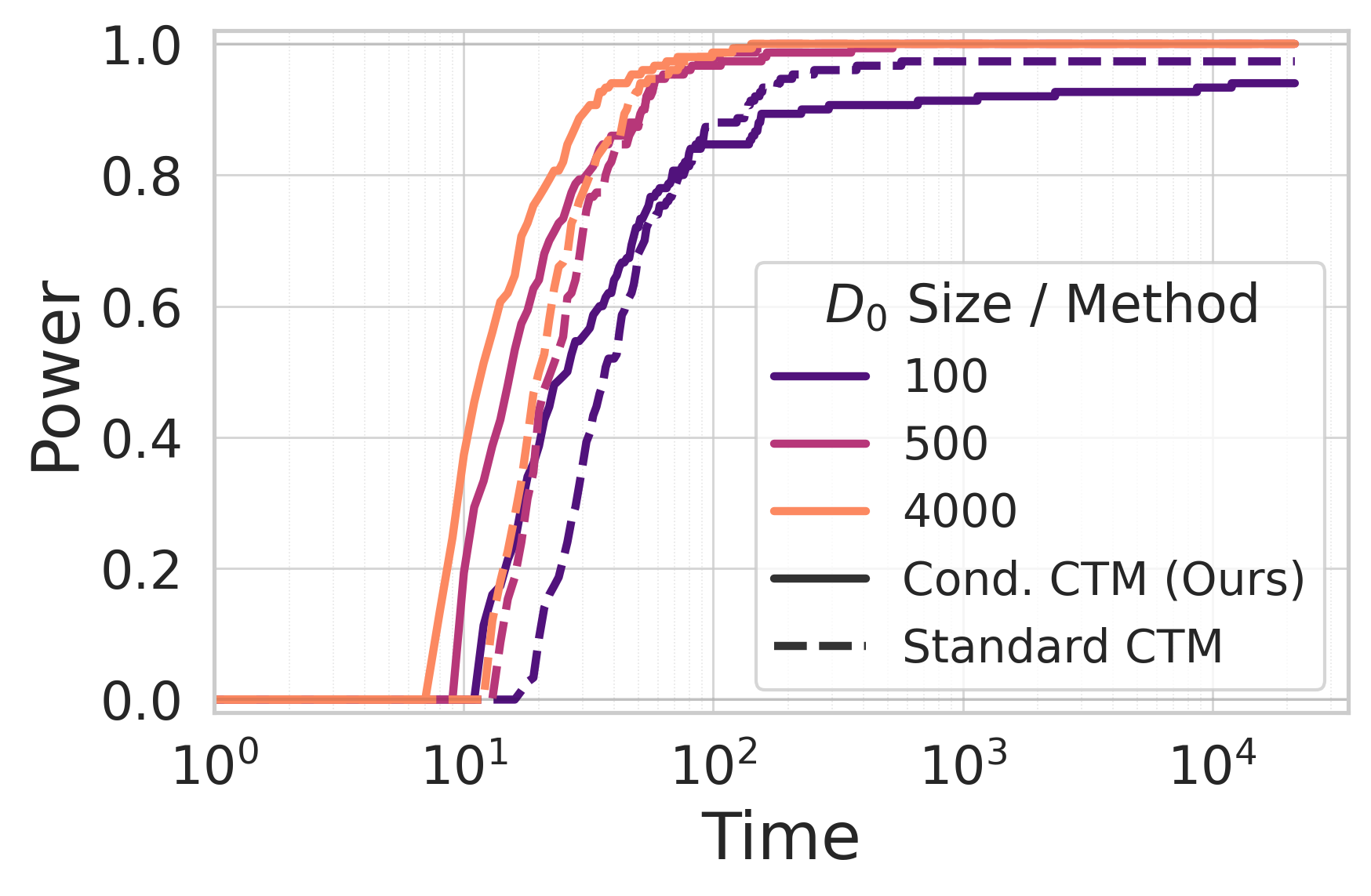}
\end{subfigure}%
\begin{subfigure}[b]{0.23\textwidth}
  \centering
  \includegraphics[width=\textwidth]{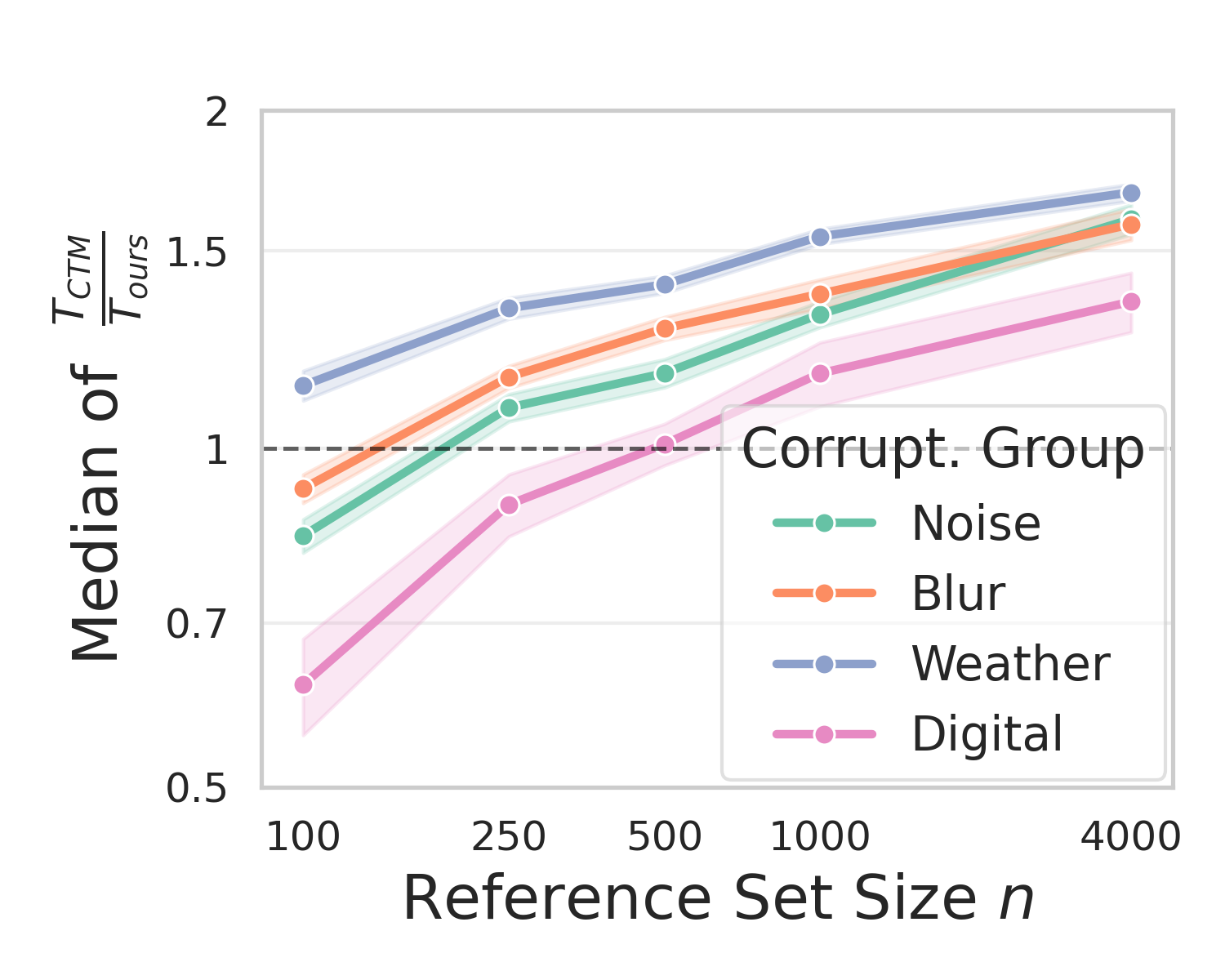}
\end{subfigure}

\caption{\textbf{Empirical power across reference-set sizes on ImageNet-C.} \textbf{Left}: power, evaluated on all 15 corruptions, as a function of the time steps. 
\textbf{Right}: median of the ratios of rejection times across different corruption groups as a function of $|D_0|=n$. Values above $1$ indicate faster detection by conditional CTM.
Shaded regions denote standard error on 10 realizations.}
\label{fig:imagenet}
\end{figure}

We begin by assessing {the effect of the reference set size on the detection power.}
Figure~\ref{fig:imagenet}~(left) reports empirical power evaluated on all 15 corruptions as a function of time steps, comparing conditional CTM to the standard CTM baseline across $|D_0|=n \in \{100,500,4000\}$. 
We can see that our conditional CTM outperforms the standard CTM as long as the reference data is of reasonable size ($n\ge 500$). 


To further examine detection delay across corruption families, we group the 15 corruptions into
Noise, Blur, Weather, and Digital. Figure~\ref{fig:imagenet}~(right) reports the
ratio of median rejection times (standard CTM / conditional CTM), where values $>1$ indicate faster detection by our
conditional CTM. For runs that do not reject by the end of the stream, we treat the detection time as the full test-set length ($37{,}500$ examples).\footnote{This affects only the Digital group at the smallest reference size ($n=100$): out of $150$ runs, $9$ of ours and $4$ of the standard CTM reached the end of the test stream without rejection.}

Following that figure, we can see that our conditional CTM rejects earlier than the standard CTM as $n$ increases, with the magnitude of this speed-up varying across different corruption groups. This demonstrates two key components that influence the effective signal strength $\Delta_k$~\eqref{eq:delta}, and hence the rejection time. First, increasing $n$ shrinks the confidence width, which directly improves the detection power. Second, the varying performance across corruption groups highlights the impact of the deviation between the alternative and the null distributions. Corruptions that induce more subtle distributional shifts (e.g., Digital) result in a smaller signal strength $\Delta_k$, naturally leading to the longer detection delays. However, for corruption types that introduce larger distributional differences, such as Blur and Weather, the resulting signal strength $\Delta_k$ is more substantial. Consequently, our method tends to achieve detection ratios above one, even with moderate reference set sizes.
\section{Discussion}

We introduced the conditional CTM, a novel framework for detecting distribution shifts that addresses the ``test-time contamination'' inherent in standard CTM approaches. Our method controls the type-I error and is supported by power analysis. We showed that our method tends to outperform baseline methods when given a reasonable number of about $500$ reference samples to form $D_0$.

As shown in our experiments, the power of the test is affected by the tightness of the confidence bands; consequently, when $D_0$ is too small, these bands widen, making the test overly conservative. We believe that recent methods that utilize synthetic data can be used to mitigate this limitation \citep{angelopoulos2023prediction,kilian2025anytime,bashari2025statistical,bashari2025synthetic}.

\sh{
In our framework, we use the ONS algorithm to adaptively learn $\eta_t$. We chose ONS since the smoothness of our loss function enables the rigorous regret analysis (Appendix~\ref{appdx:online}) required to establish our power and stopping time guarantees (Section~\ref{sec:power_analysis}). That said, sequential mixture martingales present a compelling alternative \cite{kaufmann2021mixture,waudby2024estimating}. By maintaining a predictive distribution over $\eta$, mixture martingales can adaptively learn the betting parameter. However, since proving formal power guarantees for mixture martingales under our betting function poses distinct theoretical challenges, we leave its exploration for future work.
}

Lastly, \sh{our proposed betting function targets deviations in the mean of the conformal p-values. While common and effective, this may miss deviations that do not alter the mean. Therefore,} another future direction is to extend our conditional CTM test to utilize other betting functions, with the goal of further enhancing statistical power. 


\section*{Acknowledgments and Disclosure of Funding
}
S.~S., Y.~B., and Y.~R. were supported by the European Union (ERC, SafetyBounds, 101163414). Views and opinions expressed are however those of the authors only and do not necessarily reflect those of the European Union or the European Research Council Executive Agency. Neither the European Union nor the granting authority can be held responsible for them. This research was also partially supported by the Israel Science Foundation (ISF grant 729/21).
Y.~R. acknowledges additional support from the Career Advancement Fellowship at the Technion.

\section*{Impact Statement}
This paper presents work whose goal is to advance the field of statistical testing and machine learning. There are many potential societal consequences of our work, none
which we feel must be specifically highlighted here.

\bibliography{example_paper}
\bibliographystyle{icml2026}

\newpage
\appendix
\onecolumn

\section{Experimental Details for Figure~\ref{fig:synth_pval_contamination}}
\label{appdx:synth_pval}

In this experiment, we illustrate the phenomenon of ``test-time contamination'' and its impact on the conformal p-values used for betting. The setup is designed to simulate a scenario where a sustained distribution shift becomes the ``new normal'' for the standard CTM, while our conditional CTM maintains high sensitivity.

To this end, we consider a standard normal null distribution, $P = \mathcal{N}(0,1)$. We sample a reference set $D_0$ of size $n=100$ from $P$. To simulate the data stream, we generate a sequence of test samples consisting of two distinct phases. In-distribution phase: the first 300 time steps consist of samples drawn from the null distribution $P = \mathcal{N}(0,1)$, serving as a baseline period where $H_0$ holds. Out-of-distribution phase: the subsequent $2,000$ time steps consist of samples drawn from a shifted distribution $Q = \mathcal{N}(1,1)$, representing a change-point scenario. 

We compare two strategies for computing conformal p-values on this stream: (1) the standard CTM, which updates its reference set continuously, and (2) our conditional CTM, which relies solely on the fixed reference set $D_0$. Figure~\ref{fig:synth_pval_contamination} plots a local mean of the latest 10 of the resulting conformal p-values over time. As illustrated, immediately following the shift (at $t=300$), both methods react with a spike in p-values, indicating high ``atypicality''. However, the standard CTM suffers from contamination: as the reference set accumulates observations from the shifted distribution $\mathcal{N}(1,1)$, the new data points appear increasingly ``normal'' relative to the history, causing the p-values to decay back toward the in-distribution level of 0.5. In contrast, the conditional CTM maintains consistently high p-values throughout the out-of-distribution phase, preserving the evidence of the shift.

\section{Experimental Details for Figure~\ref{fig:motivation_combined}}
\label{appdx:fig_1}

In this motivating example, we demonstrate the necessity of correctly handling the uncertainty induced by the finite reference set. Let the null distribution be a standard normal, $P = \mathcal{N}(0,1)$. We sample a reference set $D_0$ of different sizes $n$ from $P$, and subsequently simulate a stream of test samples $X_1, X_2, \dots$ also drawn from $P$ (i.e., the null hypothesis $H_0$ holds). We compare two approaches: (1) A naive CTM, which uses the ECDF $\hat{F}_0$ constructed from $D_0$ (instead of $\hat{F}_t$ as detailed in Section~\ref{sec:ctm}), computes the conformal p-values as $p_t = \hat{F}_0(X_t)$, and bets against the uniformity of these values as if they were i.i.d. uniform; (2) Our conditional CTM (detailed in Section~\ref{sec:martingale}), which also uses the fixed $D_0$, but accounts for the dependency and the estimation error arising from it. 
The left panel of Figure~\ref{fig:motivation_combined} illustrates the empirical Type-I error rate of both methods across varying reference set sizes $n$, evaluated over 100 independent trials. As observed, the Naive CTM fails to control the Type-I error, exceeding the level $\alpha$ even for large reference sets. In contrast, our proposed method achieves strict Type-I error control for all values of $n$. This demonstrates that simply plugging an empirical estimate into an online testing framework is insufficient; explicit integration of the estimation uncertainty into the betting mechanism is required to ensure validity.

To complete the motivating example, beyond ensuring validity, we also present the empirical power of our method compared to the standard CTM (as detailed in Section~\ref{sec:ctm}) under an alternative hypothesis. We repeat the experimental setup with $n=2000$, but this time, the test samples $X_1, X_2, \dots$ are drawn from the alternative distribution $\mathcal{N}(1,1)$, representing an immediate shift in the mean. The power of each test is evaluated over 100 independent trials, measuring the power up to time $t$. As illustrated in the right panel of Figure~\ref{fig:motivation_combined}, our proposed robust test achieves higher power and a faster time to rejection compared to the fully-online CTM. This superiority demonstrates that properly leveraging the fixed reference set provides empirical advantages, confirming the benefits of a stable null reference.

\section{Related Work on Testing by Betting}
\label{appdx:betting_related}
The flexibility of the testing by betting paradigm has led to its successful application in numerous sequential analysis problems. On the theoretical side, these range from one- and two-sample tests \citep{shekhar2023nonparametric}, through tests for independence and conditional independence \citep{shaer2023model, podkopaev2023sequential}, to estimating means and bounds of random variables \citep{waudby2024estimating}, and more. For practical applications, testing by betting methods have been developed for continual monitoring of model deployments \citep{vovk2021retrain, podkopaev2022tracking, prinster2025watch}, test-time adaptation \citep{barprotected, schirmer2025monitoring}, 
semantic interpretability \citep{teneggi2024testing}, monitoring risks under unknown shifts \citep{timans2025continuous}, and others \citep{chugg2023auditing, chen2025online, sadhuka2025valuator}.
Of particular relevance to our work, this framework is also used for testing for exchangeability \citep{vovk2021testing, ramdas2022testing, saha2024testing, fischer2025sequential}, and for distributional changes \citep{prinster2025watch, shin2022detectors, vovk2020testing}.

\section{Theoretical Analysis of the Online Learning of $\eta_t$}
\label{appdx:online}

In this section, we adapt the analysis from \citep{dai2025individual, chen2025online} to our setting. The primary modifications address the unique structure of our betting function, where $g_t(\hat{p}_t,\eta_t)$ is not separable into a simple product of $\eta_t$ and a term independent of it. For the sake of readability, all proofs are deferred to Appendix~\ref{appdx:proofs}.

Since $\ell_t(\eta)$~\eqref{eq:loss_ons} is a twice differentiable and 1-exp-concave function, Algorithm~\ref{alg:eta_ons_smooth} can guarantee an upper bound for the regret of choosing $\eta_t$ after $T$ time steps \citep{hazan2007logarithmic, hazan2016introduction}. Formally, the regret is defined as
\begin{equation}
    \label{eq:reg}
    R_T(\eta^*) = \sum_{t=1}^T \ell_t(\eta_t) - \sum_{t=1}^T \ell_t(\eta^*),
\end{equation}
where
\begin{align}
    \label{eq:eta_star}
    \eta^* = \arg\min_{\eta} \sum_{t=1}^T \ell_t(\eta)
\end{align}
is the best decision that minimizes the loss in hindsight, that is, the best choice of $\eta$ after observing all the data. With this in place, the regret bound is stated next. 

\begin{lemma}[Adapted from \citep{dai2025individual,chen2025online}]
    \label{lemma:ONS_regret} Let $\ell_t:[-\frac{1}{2},\frac{1}{2}]\rightarrow \mathbb{R}$ be the loss function at time $t$ as defined in \eqref{eq:loss_ons}. Algorithm~\ref{alg:eta_ons_smooth} 
    guarantees
\begin{align}
\label{eq:reg_bound_ours}
    R_T(\tilde{\eta}) \leq R_T(\eta^*) \leq 2\ln \left(1 + 4T\right) + \frac{1}{2}
\end{align}
    for any benchmark $\tilde{\eta} \in [-\frac{1}{2},\frac{1}{2}]$, where $z_t$ and $a_t$ are defined in Algorithm~\ref{alg:eta_ons_smooth}.
\end{lemma}
To allow the analysis of the power of the test under the alternative, i.e., in the presence of a distribution shift,
we now form a lower bound for the test martingale at time $T$, using the regret bound presented in Lemma~\ref{lemma:ONS_regret}.


\begin{lemma}
\label{lemma:martingale_reg}
Algorithm~\ref{alg:eta_ons_smooth} guarantees that for any benchmark $\tilde{\eta} \in [-\frac{1}{2},\frac{1}{2}]$:
\begin{align}
\ln S_T &= \ln S_T(\tilde{\eta}) - R_T(\tilde{\eta}) \\ &\geq \ln S_T(\tilde{\eta}) - 2\left(\ln(1 + 4T) + \frac{1}{4}\right), \notag
\end{align}
where 
$S_T := \prod_{t=1}^T (1+g(\hat{p}_t,\eta_t))$ and $S_T(\tilde{\eta}) := \prod_{t=1}^T (1+g(\hat{p}_t,\tilde{\eta}))$ with $g(\hat{p}_t,\eta) = \mathcal{C} \cdot (\eta \cdot (\hat{p}_t - 0.5) -\sqrt{\eta^2+k^2}\cdot \epsilon(\hat{p}_t))$.
In particular, for the benchmark $\eta^*$~\eqref{eq:eta_star} we get
\begin{align}
\ln S_T &= \ln S_T(\eta^*) - R_T(\eta^*) \\ &\geq \ln S_T(\eta^*) - 2\left(\ln(1 + 4T) + \frac{1}{4}\right).\notag 
\end{align}
\end{lemma}

The bounds presented above are fundamental to our power analysis in Section~\ref{sec:power_analysis}. Specifically, Lemma~\ref{lemma:martingale_reg} provides a crucial technical step required to establish an asymptotic power of one of the test, as well as to rigorously characterize the expected stopping time under the alternative hypothesis, as detailed in Theorem~\ref{prop:power_one}.

\section{Proofs}
\label{appdx:proofs}

\subsection{Proof of Theorem~\ref{thm:err_control}}
\begin{proof}

We begin with the following lemma, which shows that~\eqref{eq:betting} is a valid betting function, i.e., conditioned on the event that the confidence bound holds, the expected value of $b_t(\hat{p}_t)$ from \eqref{eq:betting} is bounded by $1$ under the null.
\begin{lemma}
\label{thm:val}
Let $\mathcal{F}_{t-1} := \sigma(\hat{p}_1,...,\hat{p}_{t-1})$ be the filtration generated by $\hat{p}_1,...,\hat{p}_{t-1}$.
Under the null distribution, i.e., $X_t \sim P$ for all $t\geq 1$, if $|\hat{p}_t-F(X_t)| \leq \epsilon(\hat{p}_t)$ for all $t\geq 1$ where $F$ is the true CDF of $P$, then the betting function in~\eqref{eq:betting} satisfies $\mathbb{E}_{H_0}[b_t(\hat{p}_t) \mid \mathcal{F}_{t-1}, D_0] \leq 1$. 
\end{lemma}

\begin{proof}
It suffices to show that $\mathbb{E}_{H_0}[g(\hat{p}_t,\eta_t) \mid \mathcal{F}_{t-1},D_0] \leq 0$ given that $|\hat{p}_t-F(X_t)| \leq \epsilon(\hat{p}_t)$.
Recall that $\hat{p}_t:=\hat{F}_0(X_t)$, so we rewrite $g(\hat{p}_t,\eta_t)$ as
$$
g(\hat{p}_t,\eta_t) = \mathcal{C} \cdot (\eta_t (\hat{F}_0(X_t) - F(X_t) + F(X_t) - 0.5) - \sqrt{\eta_t^2+k^2}\cdot \epsilon(\hat{p}_t)).
$$
Then, the expectation of $g(\hat{p}_t,\eta_t)$ under the null is
\begin{align}
\mathbb{E}_{H_0}[g(\hat{p}_t,\eta_t) \mid \mathcal{F}_{t-1},D_0] &= \mathcal{C} \cdot \mathbb{E}_{H_0}[\eta_t (\hat{F}_0(X_t) - F(X_t) + F(X_t) - 0.5) -\sqrt{\eta_t^2+k^2}\cdot \epsilon(\hat{p}_t) \mid \mathcal{F}_{t-1},D_0] \\ 
& \leq \mathcal{C} \cdot \mathbb{E}_{H_0}[\eta_t (\hat{F}_0(X_t) - F(X_t) + F(X_t) - 0.5) - |\eta_t| \cdot \epsilon(\hat{p}_t) \mid \mathcal{F}_{t-1},D_0]\\
&= \mathcal{C} \cdot \mathbb{E}_{H_0}[ \eta_t \cdot (\hat{F}_0(X_t) - F(X_t)) - |\eta_t|\cdot \epsilon(\hat{p}_t) \mid \mathcal{F}_{t-1},D_0]
\end{align}
where the last transition is because $\mathbb{E}_{H_0}[F(X_t) - 0.5 \mid \mathcal{F}_{t-1},D_0] = 0$, following the probability integral transform and recall that $\mathcal{C}>0$. 
Now it suffices to show that if $|\hat{F}_0(X_t) - F(X_t)| \leq \epsilon(\hat{p}_t)$ then $\mathbb{E}_{H_0}[\eta_t \cdot (\hat{F}_0(X_t) - F(X_t)) - \eta_t|\cdot \epsilon(\hat{p}_t) \mid \mathcal{F}_{t-1},D_0] \leq 0$, and we do it by splitting it into four cases and show that it is satisfied in each.

\textbf{Case 1} where $0 \leq \hat{F}_0(X_t) - F(X_t) \leq \epsilon(\hat{p}_t)$ and $\eta_t \geq 0$ :
$$
\mathbb{E}_{H_0}[\eta_t \cdot (\hat{F}_0(X_t) - F(X_t)) - |\eta_t|\cdot \epsilon(\hat{p}_t) \mid \mathcal{F}_{t-1},D_0] = \underbrace{\eta_t}_{\geq 0} \cdot \mathbb{E}_{H_0}[ \underbrace{\hat{F}_0(X_t) - F(X_t) - \epsilon(\hat{p}_t)}_{\leq 0} \mid \mathcal{F}_{t-1},D_0] \leq 0.
$$

\textbf{Case 2} where $0 \leq \hat{F}_0(X_t) - F(X_t) \leq \epsilon(\hat{p}_t)$ and $\eta_t < 0$ :
$$
\mathbb{E}_{H_0}[\eta_t \cdot (\hat{F}_0(X_t) - F(X_t)) - |\eta_t|\cdot \epsilon(\hat{p}_t) \mid \mathcal{F}_{t-1},D_0] = \underbrace{\eta_t}_{<0} \cdot \mathbb{E}_{H_0}[\underbrace{\hat{F}_0(X_t) - F(X_t) + \epsilon(\hat{p}_t)}_{\geq 0} \mid \mathcal{F}_{t-1},D_0] \leq 0.
$$

\textbf{Case 3} where $-\epsilon(\hat{p}_t) \leq \hat{F}_0(X_t) - F(X_t) \leq 0$ and $\eta_t \geq 0$ :
$$
\mathbb{E}_{H_0}[\eta_t \cdot (\hat{F}_0(X_t) - F(X_t)) - |\eta_t|\cdot \epsilon(\hat{p}_t) \mid \mathcal{F}_{t-1},D_0] = \underbrace{\eta_t}_{\geq 0} \cdot \mathbb{E}_{H_0}[\underbrace{\hat{F}_0(X_t) - F(X_t) - \epsilon(\hat{p}_t)}_{\leq 0} \mid \mathcal{F}_{t-1},D_0] \leq 0.
$$

\textbf{Case 4} where $-\epsilon(\hat{p}_t) \leq \hat{F}_0(X_t) - F(X_t) \leq 0$ and $\eta_t < 0$ :
$$
\mathbb{E}_{H_0}[\eta_t \cdot (\hat{F}_0(X_t) - F(X_t)) - |\eta_t|\cdot \epsilon(\hat{p}_t) \mid \mathcal{F}_{t-1},D_0] = \underbrace{\eta_t}_{<0} \cdot \mathbb{E}_{H_0}[\underbrace{\hat{F}_0(X_t) - F(X_t) + \epsilon(\hat{p}_t)}_{\geq 0} \mid \mathcal{F}_{t-1},D_0] \leq 0.
$$

\end{proof}

Then, by Ville's inequality, the result of Lemma~\ref{thm:val} is that for any $\alpha \in (0,1)$
$$
\mathbb{P}_{H_0}\left( \exists t \geq 1 : S_{t} \geq 1/\alpha \mid \mathcal{F}_{t-1}, D_0 \right) \leq \alpha \text{ if } |\hat{p}_t-F(X_t)| \leq \epsilon(\hat{p}_t) \text{ for all } t \geq 1.
$$
Since $\mathbb{P}_{D_0}(|\hat{p}_t-F(X_t)| \leq \epsilon(\hat{p}_t) \text{ for all } t \geq 1) \geq 1-\delta$ by construction \eqref{eq:CI}, then 
$$
\mathbb{P}_{D_0}\bigl( \mathbb{P}_{H_0}\left( \exists t \geq 1 : S_t \geq 1/\alpha \mid \mathcal{F}_{t-1}, D_0 \right) \leq \alpha \bigr) \geq 1-\delta.
$$
\end{proof}

\subsection{Proof of Lemma~\ref{lemma:ONS_regret}}
\begin{proof}

To prove Lemma~\ref{lemma:ONS_regret}, we first mention a well known result from the ONS literature \citep{hazan2016introduction}.
\begin{lemma}[\citet{hazan2016introduction}]
    \label{lemma:ONS_regret_orig} Let $\ell_t:[-\frac{1}{2},\frac{1}{2}]\rightarrow \mathbb{R}$ be the loss function at time $t$ as defined in \eqref{eq:loss_ons}. Let $G$ be a bound on the magnitude of the (sub)gradients of $\ell_t$. If $G \leq 2$, then Algorithm~\ref{alg:eta_ons_smooth}
    guarantees
    \begin{equation}
        \label{eq:bound_reg}
        R_T(\tilde{\eta}) \leq 2\sum_{t=1}^T \frac{z_t^2}{a_t} + \frac{a_0}{2}
    \end{equation}
    for any benchmark $\tilde{\eta} \in [-\frac{1}{2},\frac{1}{2}]$, where $z_t$ and $a_t$ are defined in Algorithm~\ref{alg:eta_ons_smooth}. Specifically, it holds for the best decision in hindsight $\eta^*$ as defined in~\eqref{eq:eta_star}:
        \begin{equation}
        \label{eq:bound_reg_star}
        R_T(\eta^*) \leq 2\sum_{t=1}^T \frac{z_t^2}{a_t} + \frac{a_0}{2}.
    \end{equation}
\end{lemma}

We first note that 
the bound for the magnitude of the gradient of $\ell_t(\eta)$~\eqref{eq:loss_ons} is 
$|{\C \cdot (\hat{p}_t - 0.5 -\frac{\eta_t}{\sqrt{\eta_t^2+k^2}}\cdot \epsilon(\hat{p}_t))}/{1+\C \cdot (\eta_t \cdot (\hat{p}_t - 0.5) -\sqrt{\eta_t^2+k^2}\cdot \epsilon(\hat{p}_t))}| \leq 2 = G$, 
which meets Lemma~\ref{lemma:ONS_regret_orig} assumptions.

To derive the regret's upper bound for our algorithm, we recall that we set $a_0=1$ and $\eta_t \in [-1/2,1/2]$ for all $t$. 
Next, we show that $\sum_{t=1}^{T}(z_t^2/a_t)$ is capped by an upper bound. Using the inequality $\frac{a-b}{a} \leq \ln \frac{a}{b}$, which stems from the first-order Taylor expansion of $\ln(a/b)$ where $a, b \in \mathbb{R}_{+}$, we have: 
\begin{align}
\sum_{t=1}^{T} \frac{z_t^2}{a_t} &= \sum_{t=1}^{T} \frac{1}{a_t} \cdot (a_t - a_{t-1}) \\ &\leq \sum_{t=1}^{T} \ln \left(\frac{a_t}{a_{t-1}}\right) = \ln \left(\frac{a_T}{a_0}\right). \notag
\end{align}

In our setting, $a_T = a_0 + \sum_{t=1}^{T} z_t^2$, where $a_0 = 1$, $z_t = \frac{g'_t}{1+g_t}$. We recall inequalities~\eqref{eq:bound_reg} and~\eqref{eq:bound_reg_star},
the upper bound of the regret is
\begin{align}
\label{eq:reg_bound_our}
    R_T(\tilde{\eta}) \leq R_T(\eta^*) &\leq   2\sum_{t=1}^T \frac{z_t^2}{a_t} + \frac{1}{2} a_0 \\ \notag
    &\leq 2\ln \left(\frac{a_T}{a_0}\right) + \frac{1}{2} \\ \notag
    &= 2 \ln \left(1 + \sum_{t=1}^{T} \frac{{g'_t}^2}{(1+g_t)^2}\right) + \frac{1}{2} \\ \notag
    &\leq   2\ln \left(1 + 4\sum_{t=1}^{T} {g'_t}^2\right) + \frac{1}{2} \\ \notag
    &\leq  2\ln \left(1 + 4T\right) + \frac{1}{2}. \notag
\end{align}

The first inequality is since the highest regret is for the best decision in hindsight $\eta^*$. The penultimate transition is since $\eta_t \in [-1/2, 1/2]$ then $(1 + g_t)^2 \geq \frac{1}{4}$. The last inequality is since $g'_t \in [-1,1]$ for all $t$.
\end{proof}
\subsection{Proof of Lemma~\ref{lemma:martingale_reg}}

\begin{proof}

We can express $\ln S_T$ as
\begin{align}
\label{eq:log_martingale}
\ln S_T = \sum_{t=1}^T \ln(1 + g(\hat{p}_t,\eta_t)),
\end{align}

and $\ln S_T(\tilde{\eta})$ as
\begin{align}
\label{eq:log_martingale_benchmark}
\ln S_T(\tilde{\eta}) = \sum_{t=1}^T \ln(1 + g(\hat{p}_t,\tilde{\eta})). 
\end{align}

By subtracting equation~\eqref{eq:log_martingale} from~\eqref{eq:log_martingale_benchmark} on both sides, we obtain the following
\begin{align}
\ln S_T(\tilde{\eta}) - \ln S_T &= \sum_{t=1}^T \ln(1 + g(\hat{p}_t,\tilde{\eta})) - \sum_{t=1}^T \ln(1 + g(\hat{p}_t,\eta_t)) \\
&= -\sum_{t=1}^T \ln(1 + g(\hat{p}_t,\eta_t)) - \left(-\sum_{t=1}^T \ln(1 + g(\hat{p}_t,\tilde{\eta}))\right) \\
&= \sum_{t=1}^T -\ln(1 + g(\hat{p}_t,\eta_t)) - \sum_{t=1}^T -\ln(1 + g(\hat{p}_t,\tilde{\eta})) = R_T(\tilde{\eta}).
\end{align}

The last equation followed by the definition of the regret~\eqref{eq:reg} with the loss defined in~\eqref{eq:loss_ons}. 

Consequently, we have
\begin{align}
\label{eq:martingale_reg}
\ln S_T = \ln S_T(\tilde{\eta}) - R_T(\tilde{\eta}),
\end{align}
and applying the regret bound from equation~\eqref{eq:reg_bound_ours} in Lemma~\ref{lemma:ONS_regret}  leads to the result. 
\end{proof}

\subsection{Proof of Lemma~\ref{lemma:expected_bet}}
\label{appdx:proof_lemma_expected_bet}
\begin{proof}

We first state an auxiliary lemma, which will serve us in the following proof.

\begin{lemma}
\label{lemma:log_lb}
If a random variable $Z$ satisfies $Z \geq a > 0$ almost surely, then
$$
\mathbb{E}[\ln Z] \geq \ln \mathbb{E}[Z] - \frac{\text{Var}(Z)}{2a^2}.
$$
\end{lemma}

\begin{proof}
On $[a,\infty)$, $\phi(y) = \ln y$ is concave with $\phi''(y) = -1/y^2 \geq -1/a^2$. By Taylor expansion around $m = \mathbb{E}[Z]$:
$$\phi(y) \geq \phi(m) + \phi'(m)(y-m) - \frac{1}{2a^2}(y-m)^2.$$
Setting $y = Z$ and taking expectations, the linear term $\mathbb{E}[Z - \mathbb{E}[Z]] = 0$ vanishes, yielding the result.
\end{proof}

Now, recall that $\hat{p}_t = \hat{F}_0(X_t) \in [0,1]$ and 
$$g(\hat{p},\eta) = \mathcal{C} \cdot \left(\eta \cdot (\hat{p} - 0.5) - \sqrt{\eta^2+k^2}\cdot \epsilon(\hat{p}_t)\right),$$
where $\eta \in [-\frac{1}{2}, \frac{1}{2}]$.
Define $\mu := \mathbb{E}_{H_1}[\hat{p}_t - 0.5 \mid D_0]$ and $\sigma^2 := \text{Var}(\hat{p}_t \mid D_0) \leq \frac{1}{4}$ (bounded since $\hat{p}_t \in [0,1]$). Without loss of generality, assume $\mu > 0$ (the case $\mu < 0$ follows by symmetry).

We begin by bounding $\mathbb{E}_{H_1}[\ln(1 + g(\hat{p},\eta)) \mid D_0]$ from below for a specific choice of $\eta$, and then ensure the positivity of the bound.

Since $\mu > \sqrt{1+k^2} \Bar \epsilon > \Bar \epsilon$, we can choose $\lambda > \frac{ \Bar \epsilon}{\sqrt{\mu^2- \Bar \epsilon^2}}$ and define
$$\eta := \lambda k,$$
for sufficiently small $k$ such that $\eta \in (0, \frac{1}{2})$, i.e., $k\leq \frac{1}{2\lambda}$.

Define the auxiliary quantity
$$\rho := \mu \lambda - \Bar \epsilon\sqrt{1+\lambda^2} > 0,$$
where positivity holds by our choice of $\lambda$ satisfying $\lambda^2(\mu^2 - \Bar \epsilon^2) > \Bar \epsilon^2 \Rightarrow \mu^2 \lambda^2 > \Bar \epsilon^2 + \Bar \epsilon^2 \lambda^2$.
To be able to bound $\mathbb{E}_{H_1}[\ln(1 + g(\hat{p},\eta)) \mid D_0]$, consider the worst case for $1 + g(\hat{p}, \eta)$, which occurs when $\hat{p}=0$ with $\Bar \epsilon^* := \max_{\hat{p} \in [0,1]}\Bar \epsilon(\hat{p})$, i.e., 
$$1 + g(\hat{p}, \eta) \geq 1 - \mathcal{C}\left(-\frac{1}{2} \eta - \epsilon^* \sqrt{\eta^2 + k^2} \right) = 1 - \mathcal{C} k \left(\frac{\lambda}{2} + \epsilon^*\sqrt{1+\lambda^2}\right) =: a(k).$$
Hence, if $k < k_1 := \frac{1}{2\mathcal{C}(\frac{\lambda}{2} + \epsilon^*\sqrt{1+\lambda^2})} > 0$, then $a(k) \geq \frac{1}{2}$, and we will use it hereafter when using Lemma~\ref{lemma:log_lb}. Note that $k_1 <\frac{1}{2\lambda}$, which ensures that $ \eta \in [-1/2,1/2]$.

By the construction of $\rho$ we get:
$$\mathbb{E}_{H_1}[1 + g(\hat{p},\eta) \mid D_0] = 1 + \mathcal{C} k \rho,$$
and for $k \leq k_2 := \frac{1}{\mathcal{C}\rho}>0$, we have $\mathcal{C} k \rho \in (0,1]$. Using $\ln(1+z) \geq z/2$ for $z \in (0,1]$:
$$\ln \mathbb{E}_{H_1}[1 + g(\hat{p},\eta) \mid D_0] \geq \frac{\mathcal{C} \rho k}{2}.$$

Now, for $k \leq \min\{k_1,k_2\}$ then $a(k) \geq 1/2$, and by applying Lemma~\ref{lemma:log_lb}, we get:

\begin{align}
\mathbb{E}_{H_1}[\ln(1 + g(\hat{p},\eta)) \mid D_0] &\geq \ln \mathbb{E}_{H_1}[1 + g(\hat{p},\eta) \mid D_0] - \frac{\mathcal{C}^2 \lambda^2 k^2 \sigma^2}{2 \cdot (\underbrace{a(k)}_{\geq 1/2})^2}\\
&\geq \frac{\mathcal{C} \rho k}{2} - 2\mathcal{C}^2 \lambda^2 k^2 \sigma^2,
\end{align}
where we use the fact that $\text{Var}(g(\hat{p},\eta) \mid D_0) = (\mathcal{C} \lambda k)^2 \sigma^2 = \mathcal{C}^2 \lambda^2 k^2 \sigma^2$. 

Lastly, to ensure positivity, we choose $k < k_3 := \frac{\rho}{8 \mathcal{C} \lambda^2 \sigma^2}>0$, with the convention of $k_3 = \infty$ if $\sigma^2 = 0$. Then:
$$\mathbb{E}_{H_1}[\ln(1 + g(\hat{p},\eta)) \mid D_0] \geq \frac{\mathcal{C} \rho k}{4} > 0.$$

Therefore, for all $k < k_0 := \min\{k_1, k_2, k_3\}>0$, we have constructed an $\eta \in [-1/2,1/2]$ such that $\mathbb{E}_{H_1}[\ln(1 + g(\hat{p},\eta)) \mid D_0] > 0$, which implies
$$\max_{\eta \in [-1/2,1/2]} \mathbb{E}_{H_1}[\ln(1 + g(\hat{p},\eta)) \mid D_0] > 0.$$
\end{proof}

{\color{black}
\subsection{Proof of Theorem~\ref{prop:power_one}}



\begin{proof} \ \\
\textbf{Asymptotic power one.}
Asymptotic power $\beta = 1$ means that when $H_1$ holds, our algorithm will ensure that $S_t \geq 1/\alpha$ for a given $\alpha \in (0,1)$ in finite time $t$, that is:
\begin{align}
\mathbb{P}_{H_1}(\tau = \infty \mid D_0) \leq 1 - \beta = 0,
\end{align}
where $\tau$ is the stopping time, defined as $\tau=\inf\{t\geq 1:S_t\geq 1/\alpha\}.$ We denote the domain of $\eta$ by $\mathcal{D} = [-1/2,1/2]$, and for the rest of the proof we assume $H_1$ is true and $D_0$ is given and omit the corresponding notations for readability.
Denote $\omega_* := \max_{\eta \in \mathcal{D}} \E{\ln(1 + g(\hat{p},\eta))}>0$ the expected wealth of the best benchmark in a single round, which is greater than $0$ following Lemma~\ref{lemma:expected_bet}, and by $\eta_* := \arg\max_{\eta \in \mathcal{D}} \E{\ln(1 + g(\hat{p},\eta))}$ the corresponding benchmark.
Taking the expectation on both sides of equation~\eqref{eq:martingale_reg}, we have
\begin{align}
\mathbb{E}[\ln S_t] &= \mathbb{E}[\ln S_t(\eta_*)] - \mathbb{E}[R_t(\eta_*)] \\
&= \mathbb{E}\left[\sum_{s=1}^t \ln(1 + g(\hat{p}_s,\eta_*))\right] - \mathbb{E}[R_t(\eta_*)] \\
&= t\omega_* - \mathbb{E}[R_t(\eta_*)],
\end{align}
where the last equality holds by assuming that the random variables $(g(\hat{p}_s,\eta_*))_{s \geq 1}$ are i.i.d given $D_0$.
We now analyze the probability that the null has not been rejected by time $t$, which is when the event $\{S_t < \frac{1}{\alpha}\}$ holds.
\begin{align}
\label{eq:not_rej_prob}
\mathbb{P}\left[S_t < \frac{1}{\alpha}\right] &= \mathbb{P}\left[\ln S_t < \ln\left(\frac{1}{\alpha}\right)\right] \\
&= \mathbb{P}\left[\ln S_t - \mathbb{E}[\ln S_t(\eta_*)] < \ln\left(\frac{1}{\alpha}\right) - \mathbb{E}[\ln S_t(\eta_*)]\right] \\
&= \mathbb{P}\left[\ln S_t(\eta_*) - R_t(\eta_*) - \mathbb{E}[\ln S_t (\eta_*)] < \ln\left(\frac{1}{\alpha}\right) - \mathbb{E}[\ln S_t(\eta_*)]\right] \\
&= \mathbb{P}\left[\ln S_t(\eta_*) - \mathbb{E}[\ln S_t(\eta_*)] < R_t(\eta_*) + \ln\left(\frac{1}{\alpha}\right) - t\omega_*\right],
\end{align}
where we use equation~\eqref{eq:martingale_reg} in the penultimate transition.
Now we are going to show that $R_t(\eta_*) + \ln\left(\frac{1}{\alpha}\right) - t\omega_* \leq \frac{-t\omega_*}{2}$ when $t$ is sufficiently large. Then, from equation~\eqref{eq:not_rej_prob}, we will have
\[
\mathbb{P}\left[\ln S_t(\eta_*) - \mathbb{E}[\ln S_t(\eta_*)] <R_t(\eta_*) + \ln\left(\frac{1}{\alpha}\right) - t\omega_*\right]
\leq
\mathbb{P}\left[\ln S_t(\eta_*) - \mathbb{E}[\ln S_t(\eta_*)] < -\frac{t}{2}\omega_*\right].
\]
From the regret bound of Algorithm~\ref{alg:eta_ons_smooth} in equation~\eqref{eq:reg_bound_ours}, it suffices to show that
\begin{align}
\label{eq:midterm}
2\ln\left(1 + 4t\right) + \frac{1}{2} + \ln\left(\frac{1}{\alpha}\right) \leq \frac{t}{2}\omega_*.
\end{align}
We will require the time $t$ to satisfy $t \geq \frac{2}{\omega_*}$, which leads to $\frac{1}{2} \leq \frac{t\omega_*}{4}$. Therefore, to guarantee equation~\eqref{eq:midterm}, it suffices to find $t$ such that
\begin{align}
\frac{t\omega_*}{4} \geq \ln\left(\frac{1}{\alpha}\right) + 2\ln\left(1 + 4t\right).
\end{align}
In particular, defining
\begin{equation}
\label{eq:t_cond_explicit}
t_* := \left\lceil \frac{16}{\omega_*}\ln\left(\frac{1}{\alpha\omega_*}\right)\right\rceil,
\end{equation}
one has that equation~\eqref{eq:midterm} holds for all $t\ge t_*$, since $t$ dominates $2\ln (1+4t)$ when $t$ is not small. Note that there exist such $t$ since we assume $\omega_*>0$ by Lemma~\ref{lemma:expected_bet}. Therefore, for all $t\ge t_*$,
\begin{align}
\label{eq:midterm_prob}
\mathbb{P}\left[S_t < \frac{1}{\alpha}\right] \leq \mathbb{P}\left[\ln S_t(\eta_*) - \mathbb{E}[\ln S_t(\eta_*)] < -\frac{t}{2}\omega_*\right].
\end{align}

Denote
\[
\psi_t := \ln(1+ g(\hat{p}_t,\eta_*)) - \mathbb{E}[\ln(1 + g(\hat{p}_t,\eta_*))] \in [\ln(1/2) - \omega_*, \ln(3/2) - \omega_*].
\]
We have that $\ln S_t(\eta_*) - \mathbb{E}[\ln S_t(\eta_*)] = \sum_{s=1}^t \psi_s$, which is a sum of zero-mean i.i.d bounded random variables. Denote $\sigma_*^2 := \operatorname{Var}(\psi_t)$.
By Bernstein's inequality, for any $c>0$,
\begin{align}
\label{eq:bern}
\mathbb{P}\left[\frac{1}{t}\sum_{s=1}^t \psi_s \leq -c\right]
\leq
\exp\left(
-\frac{t c^2}{2\sigma_*^2 + \frac{2}{3}(\ln 3)c}
\right).
\end{align}
Then, letting $c=\frac{\omega_*}{2}$ in~\eqref{eq:bern},
\begin{align}
\label{eq:midterm_prob_bern}
\mathbb{P}\left[\ln S_t(\eta_*) - \mathbb{E}[\ln S_t(\eta_*)] < -\frac{t}{2}\omega_*\right]
&=
\mathbb{P}\left[\frac{1}{t}\sum_{s=1}^t \psi_s < -\frac{1}{2}\omega_*\right] \\
&\leq
\exp\left(
-\frac{t\omega_*^2}{8\sigma_*^2 + \frac{4}{3}(\ln 3)\omega_*}
\right).
\end{align}
Hence, by combining equations~\eqref{eq:midterm_prob} and~\eqref{eq:midterm_prob_bern} we obtain that for all $t\ge t_*$,
\begin{align}
\label{eq:final_bern_power}
\mathbb{P}\left[S_t < \frac{1}{\alpha}\right]
\leq
\exp\left(
-\frac{t\omega_*^2}{8\sigma_*^2 + \frac{4}{3}(\ln 3)\omega_*}
\right).
\end{align}
Let $A_t$ be the event that we stop at time $t$. Then,
\begin{align}
\mathbb{P}[\tau = \infty] &= \lim_{\tau \to \infty} \mathbb{P}[\cap_{t \leq \tau} \neg A_t] \\
&\leq \lim_{\tau \to \infty} \mathbb{P}[\neg A_\tau] \\
&= \lim_{\tau \to \infty} \mathbb{P}\left[S_\tau < \frac{1}{\alpha}\right] \\
&\leq \lim_{\tau \to \infty} \exp\left(
-\frac{\tau\omega_*^2}{8\sigma_*^2 + \frac{4}{3}(\ln 3)\omega_*}
\right) = 0,
\end{align}
where the last inequality is by equation~\eqref{eq:final_bern_power}, and it equals to zero since $\omega_*>0$ by Lemma~\ref{lemma:expected_bet}. This completes the proof of a test with asymptotic power one.

\textbf{Expected stopping time.}
Now we upper-bound the expected stopping time as follows. Denote $t_*$ the time when equation~\eqref{eq:midterm_prob} starts to hold at any $t \geq t_*$. We have
\begin{align}
\label{eq:expected_tau}
\mathbb{E}[\tau] &= \sum_{t=1}^\infty \mathbb{P}[\tau > t] \\
&= \sum_{t=1}^\infty \mathbb{P}[\cap_{s \leq t} \neg A_s] \\
&\leq \sum_{t=1}^\infty \mathbb{P}[\neg A_t] \\
&= \sum_{t=1}^\infty \mathbb{P}\left[S_t < \frac{1}{\alpha}\right] \\
&\leq t_* + \sum_{t=t_*}^\infty \mathbb{P}\left[S_t < \frac{1}{\alpha}\right].
\end{align}

To proceed, we first denote by $g_k(\hat{p},\eta)$ as $g(\hat{p},\eta)$ with a specific choice of $k$. Next, we recall the notation $\Bar \epsilon = \E{\epsilon(\hat{p})}$, and since we assume $k\rightarrow 0$ then $\omega_*$ can be written as $\omega_* := \max_{\eta \in \mathcal{D}} \lim_{k \rightarrow 0} \E{\ln(1 + g(\hat{p},\eta))}>0$.

By the inequality $\ln(1 + c) \geq c - c^2$ for any $|c| \leq \frac{1}{2}$, we have
\begin{align}
 \omega_* &= \max_{\eta \in \mathcal{D}} \lim_{k \rightarrow 0}  \E{\ln(1 + g(\hat{p},\eta))}
= \max_{\eta \in \mathcal{D}}\E{\ln(1 + g_0(\hat{p},\eta))}
\geq \max_{\eta \in \mathcal{D}} \mathbb{E}\left[g_0(\hat{p},\eta) - (g_0(\hat{p},\eta))^2\right] \\
 &= \max_{\eta \in \mathcal{D}} \left[ \mathbb{E}\left[g_0(\hat{p},\eta)\right] - \text{Var}\left[g_0(\hat{p},\eta)\right] - \mathbb{E}\left[ g_0(\hat{p},\eta)\right]^2 \right]\\
 & = \max_{\eta \in \mathcal{D}} \mathcal{C}(\eta\mathbb{E}\left[\hat p-0.5\right]-|\eta|\Bar \epsilon)  - \mathcal{C}^2 \text{Var}\left[\eta(\hat{p}-0.5)-|\eta|\epsilon(\hat{p})\right] - \mathcal{C}^2\left(\eta\mathbb{E}[\hat p-0.5]-|\eta|\Bar \epsilon\right)^2.
\end{align}
where the first and last transitions are since $\lim_{k \rightarrow 0} \sqrt{\eta^2 + k^2} = \left.\left[\sqrt{\eta^2 + k^2}\right] \right|_{k=0} = |\eta|$. Now, we split it into two cases, when $\mathbb{E}\left[\hat p-0.5\right]>0$ and $\mathbb{E}\left[\hat p-0.5\right]<0$, and found a bound for each case. We denote by $\mathcal{D}_+=[0,1/2]$ and $\mathcal{D}_-=[-1/2,0]$ the non-negative and non-positive domains of $\eta$, respectively. We note that since $\Delta_0>0$ and $g(\hat{p},\eta) \in [-1,1]$, then to maximize $\E{\ln(1 + g(\hat{p},\eta))}$, $\eta$ must have the same sign as $\mathbb{E}\left[\hat p-0.5\right]$. Therefore:
\begin{itemize}
    \item \textbf{Case 1: $\mathbb{E}\left[\hat p-0.5\right]>0$}
\begin{align}
 \omega_* & = \max_{\eta \in \mathcal{D_+}} \mathcal{C}(\eta\mathbb{E}\left[\hat p-0.5\right]-\eta\epsilon)  - \mathcal{C}^2 \eta^2\text{Var}\left[\hat p-0.5-\epsilon(\hat{p})\right] - \mathcal{C}^2\left(\eta\mathbb{E}[\hat p-0.5]-\eta \epsilon\right)^2 \\
 &= \max_{\eta \in \mathcal{D_+}} \mathcal{C}(\eta|\mathbb{E}\left[\hat p-0.5\right]|-\eta\epsilon)  - \mathcal{C}^2 \eta^2\text{Var}\left[\hat p-0.5-\epsilon(\hat{p})\right] - \mathcal{C}^2\left(\eta|\mathbb{E}[\hat p-0.5]|-\eta \epsilon\right)^2\\
 & = \max_{\eta \in \mathcal{D}_+} \mathcal{C} \eta \cdot \Delta_0 - \mathcal{C}^2 \eta^2 \left(\text{Var}\left[\hat p-0.5-\epsilon(\hat{p})\right] + \Delta_0 ^2\right) \\
 &= \frac{\Delta_0^2}{4(\text{Var}[\hat p-0.5-\epsilon(\hat{p})] + \Delta_0^2)} \geq \frac{\Delta_0^2}{4(1 + \Delta_0^2)},
\end{align}
where $\text{Var}[\hat p-0.5-\epsilon(\hat{p})] \leq 1$, and the third transition is since $k \rightarrow 0$. The last transition is found by finding the maximum of a quadratic term and substituting it accordingly.

\item \textbf{Case 2: $\mathbb{E}\left[\hat p-0.5\right]<0$}
\begin{align}
 \omega_* & = \max_{\eta \in \mathcal{D_-}} \mathcal{C}(\eta\mathbb{E}\left[\hat p-0.5\right]+\eta\epsilon)  - \mathcal{C}^2 \eta^2\text{Var}\left[\hat p-0.5+\epsilon(\hat{p})\right] - \mathcal{C}^2\left(\eta\mathbb{E}[\hat p-0.5]+\eta \epsilon\right)^2 \\
 &= \max_{\eta \in \mathcal{D_-}} \mathcal{C}(-\eta|\mathbb{E}\left[\hat p-0.5\right]|+\eta\epsilon)  - \mathcal{C}^2 \eta^2\text{Var}\left[\hat p-0.5+\epsilon(\hat{p})\right] - \mathcal{C}^2\left(-\eta|\mathbb{E}[\hat p-0.5]|+\eta \epsilon\right)^2 \\
 & = \max_{\eta \in \mathcal{D}_-} -\mathcal{C} \eta \cdot \Delta_0 - \mathcal{C}^2 \eta^2 \left(\text{Var}\left[\hat p-0.5+\epsilon(\hat{p})\right] + \Delta_0 ^2\right) \\
 &= \frac{\Delta_0^2}{4(\text{Var}[\hat p-0.5+\epsilon(\hat{p})] + \Delta_0^2)} \geq \frac{\Delta_0^2}{4(1 + \Delta_0^2)},
\end{align}
where the last transitions are the same as in the first case.
\end{itemize}
Therefore, $\omega_* \geq \frac{\Delta_0^2}{4(1 + \Delta_0^2)}$. 

We now upper bound the tail term in~\eqref{eq:expected_tau} using Bernstein's inequality together with the same Case~1/Case~2 maximizer that attains the quadratic bound above. Specifically, define
\[
V :=
\begin{cases}
\operatorname{Var}\!\left[\hat p-0.5-\epsilon(\hat p)\right], & \text{if }\mathbb{E}[\hat p-0.5]>0,\\
\operatorname{Var}\!\left[\hat p-0.5+\epsilon(\hat p)\right], & \text{if }\mathbb{E}[\hat p-0.5]<0,
\end{cases}
\qquad
\mathcal D^\sharp :=
\begin{cases}
\mathcal D_+, & \text{if }\mathbb{E}[\hat p-0.5]>0,\\
\mathcal D_-, & \text{if }\mathbb{E}[\hat p-0.5]<0.
\end{cases}
\]
Let $\eta^\sharp \in \arg\max_{\eta\in\mathcal D^\sharp}\left\{\mathcal C \eta \Delta_0 - \mathcal C^2 \eta^2 (V+\Delta_0^2)\right\}$ and denote
\[
\omega^\sharp := \mathbb E\!\left[\ln\big(1+g_0(\hat p,\eta^\sharp)\big)\right].
\]
Then $\omega^\sharp = \frac{\Delta_0^2}{4(V+\Delta_0^2)}$ by the same quadratic maximization used above, and in particular $\omega^\sharp \ge \frac{\Delta_0^2}{4(1+\Delta_0^2)}$.

Next, apply the regret inequality (equation~\eqref{eq:martingale_reg}) with benchmark $\eta^\sharp$. Repeating the steps that led to~\eqref{eq:midterm_prob}, we obtain that for all $t\ge t_*^\sharp$,
\begin{align}
\label{eq:midterm_prob_sharp}
\mathbb{P}\left[S_t < \frac{1}{\alpha}\right] \leq \mathbb{P}\left[\ln S_t(\eta^\sharp) - \mathbb{E}[\ln S_t(\eta^\sharp)] < -\frac{t}{2}\omega^\sharp\right],
\end{align}
where $t_*^\sharp$ is any time such that
\begin{align}
\label{eq:midterm_sharp}
2\ln \left(1 + 4t\right) + \frac{1}{2} + \ln\left(\frac{1}{\alpha}\right) \leq \frac{t}{2}\omega^\sharp
\qquad\forall t\ge t_*^\sharp.
\end{align}
As before, a sufficient explicit choice is
\begin{equation}
\label{eq:tstar_choice_sharp}
t_*^\sharp := \left\lceil \frac{16}{\omega^\sharp}\ln\left(\frac{1}{\alpha\omega^\sharp}\right)\right\rceil.
\end{equation}

Denote
\[
\psi_t^\sharp := \ln(1+ g(\hat{p}_t,\eta^\sharp)) - \mathbb{E}[\ln(1 + g(\hat{p}_t,\eta^\sharp))].
\]
Then $\ln S_t(\eta^\sharp) - \mathbb{E}[\ln S_t(\eta^\sharp)] = \sum_{s=1}^t \psi_s^\sharp$, where $(\psi_s^\sharp)_{s\ge1}$ are i.i.d.\ mean-zero and satisfy
$|\psi_s^\sharp|\le \ln 3$ almost surely. Denote $\sigma_\sharp^2 := \operatorname{Var}(\psi_t^\sharp)$.
By Bernstein's inequality (equation~\eqref{eq:bern}), letting $c=\frac{\omega^\sharp}{2}$ yields that for all $t\ge t_*^\sharp$,
\begin{align}
\label{eq:bern_tail_exp}
\mathbb{P}\left[S_t < \frac{1}{\alpha}\right]
&\le
\exp\left(
-\frac{t(\omega^\sharp)^2}{8\sigma_\sharp^2 + \frac{4}{3}(\ln 3)\omega^\sharp}
\right).
\end{align}

Next, we upper bound $\sigma_\sharp^2$. Since $g(\hat p,\eta^\sharp)\in[-1/2,1/2]$, the function $x\mapsto \ln(1+x)$ is $2$-Lipschitz on $[-1/2,1/2]$, hence
\[
\sigma_\sharp^2
=
\operatorname{Var}\!\left(\ln(1+g(\hat p,\eta^\sharp))\right)
\le
4\,\operatorname{Var}\!\left(g(\hat p,\eta^\sharp)\right).
\]
Moreover, for the benchmark $\eta^\sharp$ defined above,
\[
\operatorname{Var}(g_0(\hat p,\eta^\sharp))
=
\mathcal C^2(\eta^\sharp)^2 V
=
\frac{\Delta_0^2 V}{4(V+\Delta_0^2)^2}
\le
\frac{\Delta_0^2}{4(V+\Delta_0^2)}
=
\omega^\sharp.
\]
Therefore,
\[
\sigma_\sharp^2 \le 4\omega^\sharp.
\]
Substituting $\sigma_\sharp^2 \le 4\omega^\sharp$ into~\eqref{eq:bern_tail_exp}, we obtain that for all $t\ge t_*^\sharp$,
\begin{align}
\label{eq:bern_tail_simplified}
\mathbb{P}\left[S_t < \frac{1}{\alpha}\right]
\le
\exp\left(
-\frac{t\omega^\sharp}{32 + \frac{4}{3}\ln 3}
\right).
\end{align}

Returning to~\eqref{eq:expected_tau} and splitting the sum at $t_*^\sharp$ yields
\begin{align}
\mathbb{E}[\tau]
&\le t_*^\sharp + \sum_{t=t_*^\sharp}^\infty \mathbb{P}\left[S_t < \frac{1}{\alpha}\right]
\le t_*^\sharp + \sum_{t=t_*^\sharp}^\infty \exp\left(
-\frac{t\omega^\sharp}{32 + \frac{4}{3}\ln 3}
\right).
\end{align}
Let $\lambda := \frac{\omega^\sharp}{32 + \frac{4}{3}\ln 3}$. Then the series above is geometric with ratio $e^{-\lambda}$ and
\[
\sum_{t=t_*^\sharp}^\infty e^{-\lambda t}
=
\frac{e^{-\lambda t_*^\sharp}}{1-e^{-\lambda}}
\le
\frac{1}{1-e^{-\lambda}}.
\]
Using $\exp(x)\ge 1+x$ for $x\ge0$ gives $e^{-\lambda}\le \frac{1}{1+\lambda}$ and thus
\[
1-e^{-\lambda}\ge \frac{\lambda}{1+\lambda},
\qquad
\frac{1}{1-e^{-\lambda}} \le \frac{1+\lambda}{\lambda} = \frac{1}{\lambda}+1.
\]
Therefore,
\begin{align}
\label{eq:Etau_intermediate}
\mathbb{E}[\tau]
&\le t_*^\sharp + \left(\frac{1}{\lambda}+1\right)
= t_*^\sharp + \left(32 + \frac{4}{3}\ln 3\right)\frac{1}{\omega^\sharp} + 1.
\end{align}

Finally, combining equation~\eqref{eq:tstar_choice_sharp} with~\eqref{eq:Etau_intermediate} and using the bound
\[
\omega^\sharp \ge \frac{\Delta_0^2}{4(1+\Delta_0^2)},
\]
we get that under the assumption of $k \rightarrow 0$:
\begin{align}
\mathbb{E}[\tau]
=
\mathcal{O}\left(\frac{1}{\Delta_0^2} \ln\left(\frac{1}{\alpha\Delta_0}\right) + \frac{1}{\Delta_0^2}\right).
\end{align}
\end{proof}

}
\section{The Effect of the Clipping Parameter}
\label{appdx:clipping}
To analyze the effect of the clipping parameter $C$ on the proposed test, we apply  our proposed test and the standard CTM with different values of $C$. We analyze both immediate change-point and delayed change-point scenarios, where we follow the experimental protocol for each scenario as detailed in Section~\ref{sec:synth}. 

Figure~\ref{fig:clipping} illustrates the power for both methods, where the left panel is for a delayed change-point of 200 samples, and the right panel is for an immediate one. As can be seen, without clipping ($C=0$), both methods suffer from wealth decay, resulting in lower power. However, higher clipping threshold ($C=0.2$) also degrades performance, as it prevents the martingale from capitalizing on the early stages of the shift. The power peaks at low non-zero values, with $C=0.1$ achieving the highest power and fastest detection across both scenarios. This justifies our choice of $C=0.1$,  which effectively balances the need to prevent wealth bleed while preserving sensitivity to shifts detection.

\begin{figure}[h]
\centering

\begin{subfigure}[b]{0.6\textwidth} 
    \includegraphics[width=\textwidth]{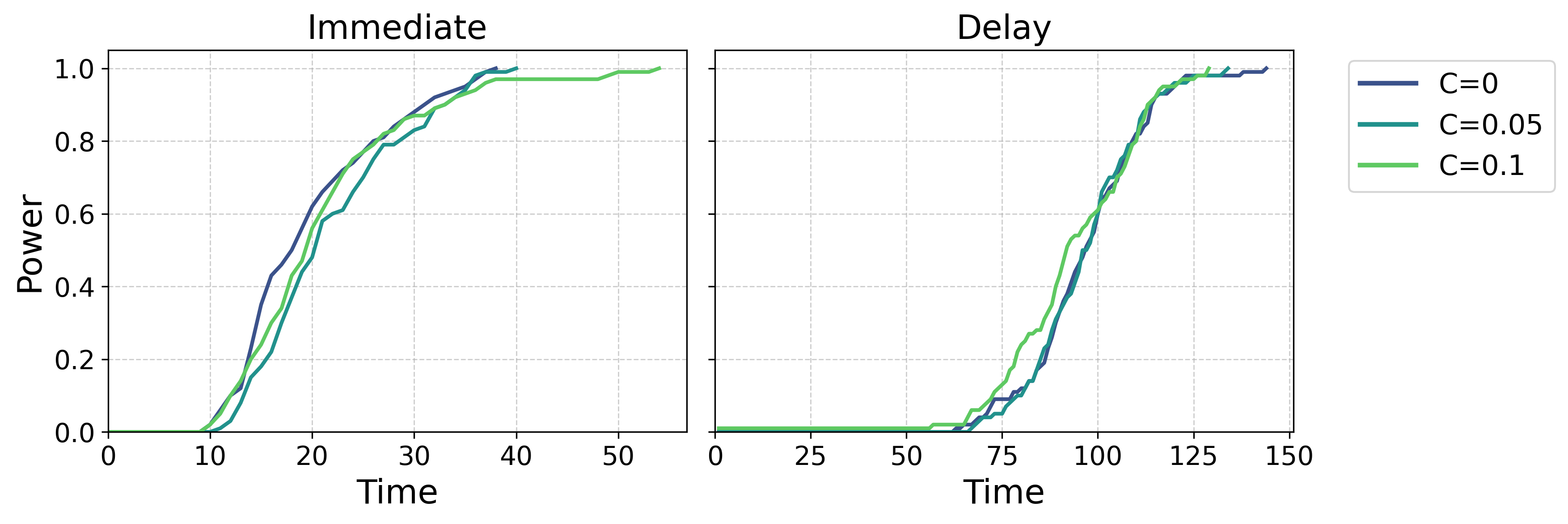}
\end{subfigure}

\caption{\textbf{Ablation study on the effect of the clipping parameter $C$.} Empirical power, evaluated over 100 repetitions, of our proposed method for different values of the clipping parameter $C$. \textbf{Left}: immediate change-point scenario. \textbf{Right:} a delayed change-point scenario, with a delay of $200$ samples. The results are presented after the delay occurs.
}
\label{fig:clipping}
\end{figure}

\section{Additional Figures for The Synthetic Experiments}
\label{appdx:synth_figs}

\begin{figure}[H]
\centering
\begin{subfigure}[b]{0.4\textwidth} 
    \includegraphics[width=\textwidth]{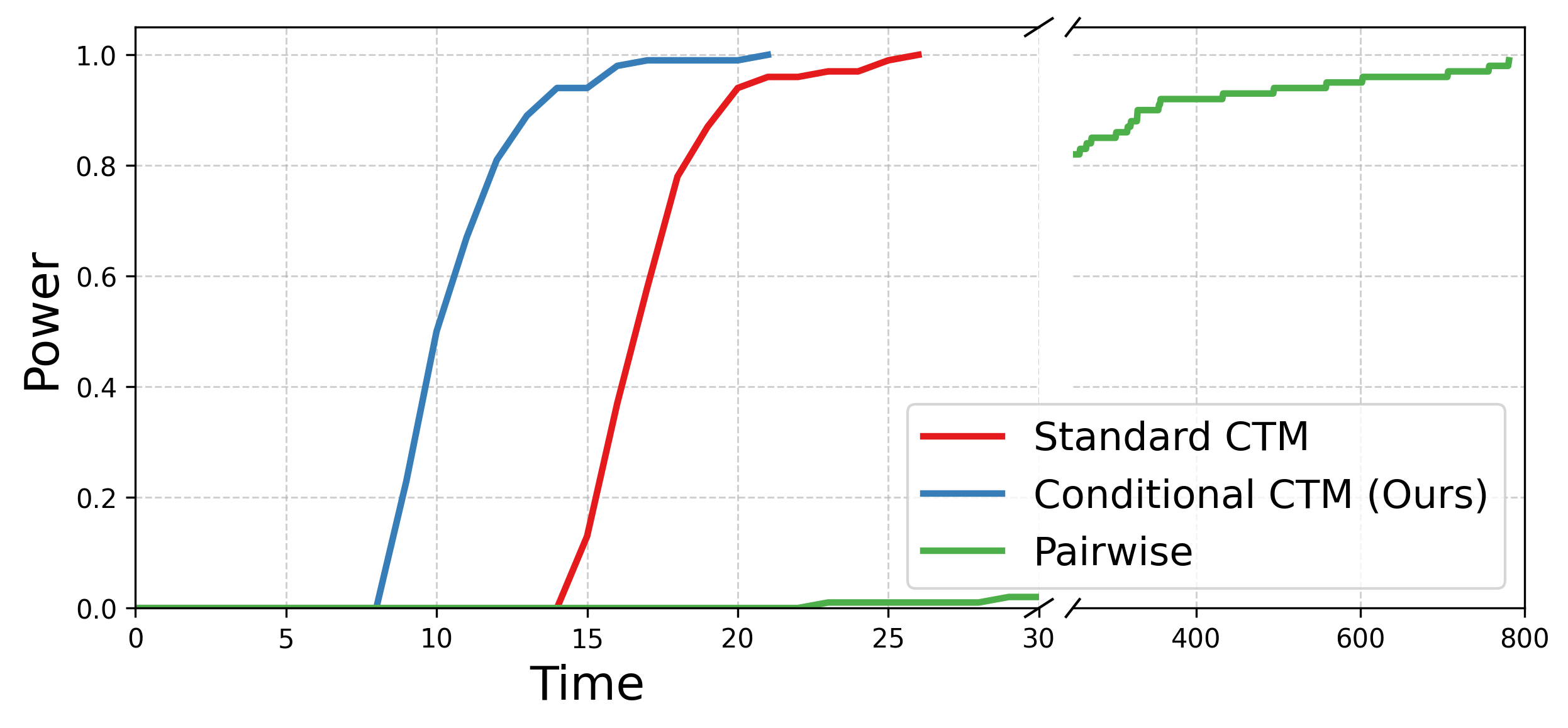}
    \label{fig:pairwise_sub}
\end{subfigure}

\caption{\textbf{Comparison to the pairwise betting approach}: empirical power, evaluated over 100 trials, of our proposed method, the standard CTM, and the pairwise betting approach \cite{saha2024testing}. The the time axis split to visualize the pairwise performance.}
\label{fig:pairwise}
\end{figure}

\section{Implementation Details of Experiments on ImageNet-C}
\label{sec:implementation}

\textbf{Model and datasets} For all experiments involving the ImageNet-C benchmark, we use a pre-trained Vision Transformer (\texttt{vit\_base\_patch16\_224})  model sourced from the \texttt{timm}. To simulate distribution shifts efficiently, we save the pre-computed entropy values for both the in-distribution holdout set and the shifted data, rather than performing raw image inference during the test-time adaptation phase. The shift experiments cover all 15 corruption types defined in the ImageNet-C benchmark, including noise, blur, weather, and digital corruptions.

\textbf{Algorithm hyperparameters} For all Imagenet-C experiments, we use a warmup phase of $N_{\text{warm}}=50$ examples that are used to initialize the betting procedure, updating internal state but not placing active bets, for both our method and the standard CTM. This warmup stabilizes the methods and reduces premature rejections when evidence is still weak. Unless stated otherwise for specific ablation studies, the default hyperparameters for our betting martingale (conditional CTM) are set as follows: 
\begin{itemize} 
    \item The significance level is set to $\alpha=0.05$. 
    \item For the Online Newton Step (ONS) optimizer, we set the diameter parameter $D=0.5$ and the constant $C=0.05$. 
    \item The smoothing parameter for the betting function is set to $10^{-6}$. 
\end{itemize}

\textbf{Experimental protocol} Besides $D_0$ which is available offline, each method processes one example at a time sequentially. After the warmup phase is done, updating the martingale values immediately after each sample. We run the test on $37,500$ examples for each test sequence. To ensure reproducibility, we explicitly seed the random number generators for Python, NumPy, and PyTorch. We report results aggregated over 10 seeds to account for variability in data ordering.
\end{document}